\documentclass[envcountsame]{lipics-v2019}

\usepackage{url}

\usepackage{wrapfig}
\usepackage{tikz}
\usetikzlibrary{calc,matrix,arrows,shapes,automata,backgrounds,petri,decorations.pathreplacing}

\newcounter{edge}
\usepackage{scrhack}
\usepackage{amsmath}
\usepackage{amssymb}
\usepackage{mathtools}
\usepackage{graphicx}
\usepackage{amsthm}
\usepackage{pdfpages}
\usepackage{color}
\usepackage[]{algorithm2e}
\usepackage{float}
\usepackage[colorinlistoftodos,textsize=tiny,color=orange!70
,disable
]{todonotes}
\usepackage{tabularx,ragged2e}
\usepackage{xfrac}
\usepackage{cutwin}
\usepackage{lmodern}
\usepackage{subcaption}
\usepackage{thmtools, thm-restate}
\theoremstyle{plain}
\newtheorem{defi}{Definition}

\newtheorem{algo}[defi]{Algorithm}

\def\a{\ensuremath\mathrm{a}}
\def\b{\ensuremath\mathrm{b}}
\def\c{\ensuremath\mathrm{c}}
\def\d{\ensuremath\mathrm{d}}
\def\e{\ensuremath\mathrm{e}}
\def\x{\ensuremath\mathrm{x}}
\def\y{\ensuremath\mathrm{y}}
\def\z{\ensuremath\mathrm{z}}

\def\X{\ensuremath\mathbf{x}}
\def\Y{\ensuremath\mathbf{y}}
\def\Z{\ensuremath\mathbf{z}}

\newcommand{\bitvec}[1]{\mathbf{#1}}
\newcommand{\ev}{\textit{ev}}

\newcommand{\out}{\mathrm{out}}
\newcommand{\leftdot}[1]{\prescript{\bullet}{}{#1}}
\newcommand{\rightdot}[1]{{#1}^\bullet}

\newcommand{\id}{\mathrm{id}}
\newcommand{\Id}{\mathrm{Id}}

\newcommand{\failop}{\mathrm{fail}}
\newcommand{\norm}{\mathrm{norm}}

\newcommand{\nat}{\mathbb{N}}

\newcommand{\heading}[1]{\subsection{#1}}

\makeatletter
\newsavebox{\@brx}
\newcommand{\llangle}[1][]{\savebox{\@brx}{\(\m@th{#1\langle}\)}%
  \mathopen{\copy\@brx\kern-0.5\wd\@brx\usebox{\@brx}}}
\newcommand{\rrangle}[1][]{\savebox{\@brx}{\(\m@th{#1\rangle}\)}%
  \mathclose{\copy\@brx\kern-0.5\wd\@brx\usebox{\@brx}}}
\makeatother

\newcommand{\short}[1]{}
\newcommand{\full}[1]{#1}

\newcommand{\cbox}[1]{\vspace{0.2cm}\noindent
  \fbox{\parbox{.97\textwidth}{#1}}\vspace{0.2cm}}

\graphicspath{{./images/}}
\newcolumntype{C}{>{\Centering\arraybackslash}X}
\newcommand{\exclude}[1]{{}}

\title{Uncertainty Reasoning for Probabilistic Petri Nets via Bayesian
  Networks}

\author{Rebecca Bernemann}{University of Duisburg-Essen, Germany}{rebecca.bernemann@uni-due.de}{}{}
\author{Benjamin Cabrera}{University of Duisburg-Essen, Germany}{benjamin.cabrera@uni-due.de}{}{}
\author{Reiko Heckel}{University of Leicester, UK}{rh122@leicester.ac.uk}{}{}
\author{Barbara König}{University of Duisburg-Essen,
  Germany}{barbara\_koenig@uni-due.de}{}{}

\funding{This work was supported by the Deutsche
  Forschungsgemeinschaft (DFG) under grant GRK 2167, Research Training
  Group ``User-Centred Social Media''.}

\authorrunning{R. Bernemann and B. Cabrera and R. Heckel and B. König}

\Copyright{Rebecca Bernemann, Benjamin Cabrera, Reiko Heckel, and
  Barbara König}



\nolinenumbers

\short{
\relatedversion{A full version of the paper is available as
  \cite{bchk:uncertainty-reasoning-arxiv},
  \url{https://arxiv.org/abs/????.?????}.}
}

\EventEditors{Nitin Saxena and Sunil Simon}
\EventNoEds{2}
\EventLongTitle{40th IARCS Annual Conference on Foundations of Software Technology and Theoretical Computer Science (FSTTCS 2020)}
\EventShortTitle{FSTTCS 2020}
\EventAcronym{FSTTCS}
\EventYear{2020}
\EventDate{December 14--18, 2020}
\EventLocation{BITS Pilani, K K Birla Goa Campus, Goa, India (Virtual Conference)}
\EventLogo{}
\SeriesVolume{182}
\ArticleNo{42}

\begin{document}

\sloppy

\maketitle
\pagenumbering{roman}
\setcounter{page}{2}

\pagenumbering{arabic}

\begin{abstract}
  This paper exploits extended Bayesian networks for uncertainty
  reasoning on Petri nets, where firing of transitions is
  probabilistic. In particular, Bayesian networks are used as symbolic
  representations of probability distributions, modelling the
  observer's knowledge about the tokens in the net. The observer
  can study the net by monitoring successful and failed steps.

  An update mechanism for Bayesian nets is enabled by relaxing some of
  their restrictions, leading to modular Bayesian nets that can
  conveniently be represented and modified.  As for every symbolic
  representation, the question is how to derive information -- in this
  case marginal probability distributions -- from a modular Bayesian
  net. We show how to do this by generalizing the known method of
  variable elimination.
  The approach is illustrated by examples about the spreading of
  diseases (SIR model) and information diffusion in social
  networks. We have implemented our approach and provide runtime
  results.
  \todo[inline]{\textbf{FSTTCS Rev1:} it seems that the 
  contribution of [4] and of this paper together is more 
  on improvement for MBN techniques than for stochastic nets themselves. Presented this way, it is not sure that this work will find an appropriate audience at FSTTCS.}
\end{abstract}

\section{Introduction}
\label{sec:introduction}

Today's software systems and the real-world processes they support are often distributed, with agents acting independently based on their own local state but without complete knowledge of the global state. E.g., a social network may expose a partial history of its users' interactions while hiding their internal states. An application tracing the spread of a virus can record test results but not the true  infection state of its subjects. Still, in both cases, we would like to derive knowledge under uncertainty to allow us, for example, to predict the spread of news in the social network or trace the outbreak of a virus.

Using Petri nets as a basis for modelling concurrent systems, 
our aim is to perform uncertainty reasoning on Petri nets, employing
Bayesian networks as compact representations of probability
distributions.
\todo[inline]{\textbf{CONCUR Rev1:} The paper completely misses the motivation of the topic.
Also, there are several issues with the descriptions.\\
- Most of the additional text in brackets can also be written as normal text.\\
- Several definitions are adapted from the previous paper but these modifications 
do not become clear from the text.\\
Overall the paper requires a substantial amount of preconception in the used concepts.
An early single example for all concepts and computation steps would help the reader 
a lot to understand the used concepts.}
\todo[inline]{\textbf{Ba:} Start with applications and motivate our
  contribution from there.}
%
%
Assume that we are observing a discrete-time concurrent system
modelled by a Petri net.  The net's structure is known, but its
initial state is uncertain, given only as an a-priori probability
distribution on markings.  The net is probabilistic: Transitions are
chosen at random, either from the set of enabled transitions or
independently, based on probabilities that are known but may change
between steps.  We cannot observe which transition actually fires, but
only if firing was successful or failed.  Failures occur if the chosen
transition is not enabled under the current marking (in the case where
we choose transitions independently), if no transition can fire, or if
a special fail transition is chosen.
After observing the system for a number
of steps, recording a sequence of ``success'' and ``failure''
events, we then determine a marginal distribution on the markings
(e.g., compute the probability that a given place is marked), taking into account
all observations.
\todo{\textbf{FSTTCS Rev1:} the semantics of stochastic nets is unusual : transitions are chosen at random and may or may not be firable, leading to a success or a failure, which is the only 
information available to an observer. Where does this semantics originate from, and what is the justification for it ?}

First, we set up a framework for uncertainty reasoning based on
time-inhomogeneous Markov chains that formally describes this
scenario, parameterized over the specific semantics of the
probabilistic net. This encompasses the well-known stochastic Petri
nets \cite{m:stochastic-petri}, as well as a semantics where the
choice of the marking and the transition is
independent. \short{(Sct.~\ref{sec:markov-ce-nets}
  and~\ref{sec:uncertainty-reasoning})}
Using basic Bayesian reasoning (reminiscent of methods used for
hidden Markov models \cite{r:hidden-markov-models}), it is
conceptually relatively straightforward to update the probability
distribution based on the acquired knowledge. However, the probability
space is exponential in the number of places of the net and hence
direct computations become infeasible relatively quickly.

Following \cite{chhk:update-ce-nets-bayesian}, our solution is to use
(modular) Bayesian networks
\cite{wbk:BN-app,d:modeling-reasoning-bn,p:bayesian-networks} as
compact symbolic representations of probability distributions. Updates
to the probability distribution can be performed very efficiently on
this data structure, simply by adding additional nodes. By analyzing
the structure of the Petri net we ensure that this node has a minimal
number of connections to already existing
nodes. \short{(Sct.~\ref{sec:bn-props} and~\ref{sec:updating-bn})}

As for every symbolic representation, the question is how to derive
information, in this case marginal probability distributions. We solve
this question by generalizing the known method of variable elimination
\cite{d:bucket-elimination,d:modeling-reasoning-bn} to modular
Bayesian networks. This method is known to work efficiently for
networks of small treewidth, a fact that we experimentally verify in
our implementation. \short{(Sct.~\ref{sec:variable-elimination}
  and~\ref{sec:implementation})}

We consider some small application examples modelling gossip and
infection spreading.

Summarized, our contributions are:

\begin{itemize}
\item We propose a framework for uncertainty reasoning based on
  time-inhomogeneous Markov chains, parameterized over different types
  of probabilistic Petri nets (Sct.~\ref{sec:markov-ce-nets}
  and~\ref{sec:uncertainty-reasoning}).

\item We use modular Bayesian networks to symbolically represent and
  update probability distributions (Sct.~\ref{sec:bn-props}
  and~\ref{sec:updating-bn}).

\item We extend the variable elimination method to modular Bayesian
  networks and show how it can be efficiently employed in order to
  compute marginal distributions
  (Sct.~\ref{sec:variable-elimination}). This is corroborated by our
  implementation and runtime results (Sct.~\ref{sec:implementation}).
\end{itemize}

\full{All proofs and further material can be found in the
  appendix.}\short{All proofs and further material can be found in the
  full version \cite{bchk:uncertainty-reasoning-arxiv}.}

\section{Markov Chains and Probabilistic Condition/Event Nets}
\label{sec:markov-ce-nets}



\heading{Markov Chains}

Markov chains \cite{GS:Markov,t:introduction-markov-chains} are a
stochastic state-based model, in which the probability of a transition
depends only on the state of origin. Here we restrict to a finite
state space.

\begin{defi}[Markov chain]
  Let $Q$ be a finite state space. A \emph{(discrete-time) Markov
    chain} is a sequence $(X_n)_{n\in\nat_0}$ of random variables such
  that for $q,q_0,\dots,q_n\in Q$:
  \[ P(X_{n+1}=q\mid X_n=q_n) = P(X_{n+1}=q\mid
    X_n=q_n,\dots,X_0=q_0). \] 
\end{defi}

Assume that $|Q|=k$. Then, the probability distribution over $Q$ at
time~$n$ can be represented as a $k$-dimensional vector $p^n$, indexed
over $Q$. We abbreviate $p^n(q) = P(X_n=q)$.  We define
$k\times k$-transition matrices $P^n$, indexed over $Q$, with
entries\footnote{We are using the notation $M(q'\mid q)$, resembling
  conditional probability,} for the entry of matrix $M$ at row~$q'$
and column~$q$:  $P^n(q'\mid q) = P(X_{n+1}=q'\mid X_n=q)$. Note that
$p^{n+1} = P^n\cdot p^n$.
We do not restrict to time-homogeneous Markov chains where it is
required that $P^n = P^{n+1}$ for all $n\in\nat_0$. Instead, the
probability distribution on the transitions might vary over time.







\heading{Probabilistic Condition/Event Nets}

As a basis for probabilistic Petri nets we use the following variant
of condition/event nets \cite{r:petri-nets}. Deviating from
\cite{r:petri-nets}, we omit the initial marking and furthermore the
fact that the post-condition is marked is not inhibiting the firing of
a transition. That is, we omit the so-called contact condition, which
makes it easier to model examples from application scenarios where the
contact condition would be unnatural. Note however that we could
easily accommodate the theory to include this condition, as we did in
the predecessor paper \cite{chhk:update-ce-nets-bayesian}.

\begin{defi}[condition/event net]
  \label{def:petri_net}
  \todo{\textbf{CONCUR Rev2:} (Def2/3) How/where is \$ m \textbackslash not
    \textbackslash rightarrow \textasciicircum t \$ defined? Ba: done}  A
  \emph{condition/event net} (C/E net or simply Petri net)
  $N = (S, T, \leftdot{()}, \rightdot{()})$ is a four-tuple consisting
  of a finite set of \emph{places} $S$, a finite set of
  \emph{transitions} $T$ with \emph{pre-conditions}
  $\leftdot{()}: T \rightarrow \mathcal{P}(S)$ and
  \emph{post-conditions}
  $\rightdot{()}: T \rightarrow \mathcal{P}(S)$.  A \emph{marking} is
  any subset of places \(m\subseteq S\) and will also be represented
  by a bit string $m\in \{0,1\}^{|S|}$ (assuming an ordering on the
  places).

  A transition $t$ can \emph{fire} for a marking $m \subseteq S$ if
  $\leftdot{t} \subseteq m$. Then marking $m$ is transformed into
  $m' = (m \setminus \leftdot{t}) \cup \rightdot{t}$, written
  $m \overset{t}{\Rightarrow} m'$.
  We write $m \overset{t}{\Rightarrow}$ to indicate that there exists
  some $m'$ with $m \overset{t}{\Rightarrow} m'$ and
  $m\not\overset{t}{\Rightarrow}$ if this is not the case. 
  We denote the \emph{set of all markings} by
  $\mathcal{M} = \mathcal{P}(S)$.
\end{defi}
\todo{\textbf{FSTTCS Rev1:} the semantics of your nets is given as $m'= m \setminus  pre(t) \cup post(t)$. What happens if some place p is marked in $m \setminus  pre(t)$ and appears in post(t) ? 
In [4] the interpretation is that t cannot fire and this leads to a fail state. Here this case is not addressed. Note that m' is defined even in this case, so it seems that 
a marked place in post(t) does not forbid t, and remains marked with a single token.  
The standard approach if you want to keep binary markings is to use an
elementary semantics for nets, that forbids t. Ba: done}

In order to obtain a Markov chain from a C/E net, we need the
following data: given a marking $m$ and a transition $t$, we denote by
$r_n(m,t)$ the probability of firing $t$ in marking $m$ (at step~$n$),
and by $r_n(m,\failop)$ the probability of going directly to a fail state
$*$.

\begin{defi}
  Let $N = (S, T, \leftdot{()}, \rightdot{()})$ be a condition/event
  net and let $T_f = T\cup\{\failop\}$ (the set of transitions
  enriched with a fail transition). Furthermore let
  $r_n\colon \mathcal{M}\times T_f\to [0,1]$, $n\in\nat_0$ be a family
  of functions (the transition distributions at step~$n$), such that
  for each $n\in\nat_0$, $m\in\mathcal{M}$:
  $\sum_{t\in T_f} r_n(m,t) = 1$.

  The \emph{Markov chain generated from $N, r_n$} has states
  $Q = \mathcal{M}\cup \{*\}$ and for $m,m'\in\mathcal{M}$:
  \begin{center}
    $P(X_{n+1} = m' \mid X_n = m) = \sum_{t\in
      T,m\overset{t}{\Rightarrow} m'} r_n(m,t)$ \qquad
    $P(X_{n+1} = m' \mid X_n = *) = 0$ \\
    $P(X_{n+1} = * \mid X_n = m) = \sum_{t\in
      T_f,m\not\overset{t}{\Rightarrow}} r_n(m,t)$ \qquad
    $P(X_{n+1} = * \mid X_n = *) = 1$
  \end{center}
  \todo{\textbf{FSTTCS Rev1:} This sentence is ambiguous. From the
    former explanations, failure occurs when a transition is selected
    and is not firable from the current marking. So fail can occur as
    soon as some transition is not firable from the set of
    markings. There must be a mistake. Ba: done}
  where we assume that $m\not\overset{\failop}{\Rightarrow}$
  \todo{\textbf{CONCUR Rev1:} \textbackslash not \textbackslash rightarrow not clear}
  for every $m\in\mathcal{M}$.
\end{defi}
\todo[inline]{\textbf{FSTTCS Rev1:} is there a reason for your
  success/failure semantics ? Does it simplify the MBN approach
  described later in the paper ? Or is it just a design choice ?}

Note that we can make a transition from $m$ to the fail state $*$
either when there is a non-zero probability for performing such a
transition directly or when we pick a transition that cannot be fired
in $m$. Requiring that $m\not\overset{\failop}{\Rightarrow}$ for every
$m$ is for notational convenience, since we have to sum up all
probabilities leading to the fail state $*$ to compute
$P(X_{n+1}=*\mid X_n=m)$. In this way the symbol $\not\Rightarrow$
always signifies a transition to $*$.

By parametrising over $r_n$ we obtain different semantics for
condition/even nets. In particular, we consider the following two
probabilistic semantics, both based on probability distributions
$p^n_T\colon T\to [0,1]$, $n\in\nat_0$ \todo{\textbf{FSTTCS Rev1:}
  Where does $p^n_T$ originate from? Is it fixed by the environment
  or do we have $p^n_T$ identical for every n? Ba: done}  on transitions. We
work under the assumption that this information is given or can be
gained from extra knowledge that we have about our environment.

\smallskip

\noindent\textsf{\textbf{Independent case:}} Here we assume that the
marking and the transition are drawn independently, where markings are
distributed according to $p^n$ and transitions according to
$p^n_T$. It may happen that the transition and the marking do not
``match'' and the transition cannot fire. Formally,
$r_n(m,t) = p^n_T(t)$, $r_n(m,\failop) = 0$ (where $m\in \mathcal{M}$,
$t\in T$). This extends to the case where
$\failop$ has non-zero probability, with probability
distribution $p^n_T\colon T_f\to [0,1]$.

\noindent\textsf{\textbf{Stochastic net case:}} We consider stochastic Petri nets \cite{m:stochastic-petri} which are
often provided with a semantics based on continuous-time Markov chains
\cite{t:introduction-markov-chains}. Here, however we do not consider
continuous time, but instead model the embedded discrete-time Markov
chain of jumps that abstracts from the timing. The firing rate of a
transition $t$ is proportional to $p^n_T(t)$.

Intuitively, we first sample a marking $m$ (according to $p^n$) and
then sample a transition, restricting to those that are enabled in
$m$.
Formally, for every $t\in T_f$, $r_n(m,t) = 0$, $r_n(m,\failop) = 1$
if no transition can fire in $m$ and
$r_n(m,t) = p^n_T(t)/\sum_{\text{$m\overset{t'}{\Rightarrow}$}}
p^n_T(t')$, $r_n(m,\failop) = 0$ otherwise.

\medskip

Other semantics might make sense, for instance the probability of
firing a transition could depend on a place not contained in its
pre-condition. Furthermore, it is possible to mix the two semantics
and do one step in the independent and the next in the stochastic
semantics.

\begin{figure}
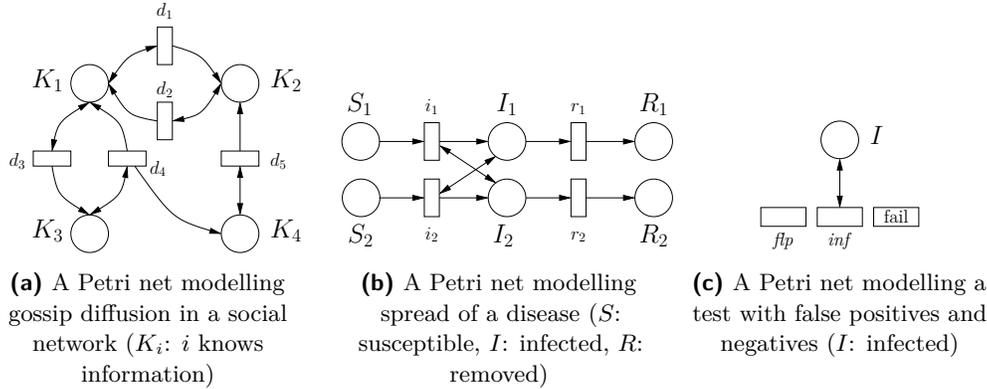

  \captionsetup[subfigure]{justification=centering}
  \begin{subfigure}[t]{0.3\textwidth}
    \centering \scalebox{0.5}{\input{net-gossip-x.tex}}
    \caption{A Petri net modelling gossip diffusion in a social
      network ($K_i$: $i$ knows information)}
    \label{fig:pn-gossip}
  \end{subfigure}
  \quad
  \begin{subfigure}[t]{0.3\textwidth}
    \centering
    \scalebox{0.5}{\input{net-sir-x.tex}}
    \caption{A Petri net modelling spread of a disease ($S$:
      susceptible, $I$: infected, $R$: removed)}
    \label{fig:pn-sir}
  \end{subfigure}
  \quad
  \begin{subfigure}[t]{0.28\textwidth}
    \centering
    \scalebox{0.5}{\input{net-test-x.tex}}
    \caption{A Petri net modelling a test with false positives and
      negatives ($I$: infected)}
    \label{fig:pn-test}
  \end{subfigure}
  \caption{Example Petri nets}
\end{figure}
\todo{\textbf{CONCUR Rev1:} Different examples are used in different 
definitions to exemplify the concepts (e.g. Example 1a and 2).
This makes it hard for the reader to follow, and a single example 
would help a lot.}

\begin{example}
  \label{ex:petri-nets}
  The following nets illustrate the two semantics. The first net
  (Fig.~\ref{fig:pn-gossip}) explains the diffusion of gossip in a
  social network: There are four users and each place $K_i$ represents
  the knowledge of user~$i$. To convey the fact that user $i$ knows
  some secret, place $K_i$ contains a token.  The diffusion of
  information is represented by transitions $d_j$.
  E.g., if $1$ knows the secret he
  will tell it to either $2$ or $3$ and if $3$ knows a secret she will
  broadcast it to both $1$ and $4$. Note that a
  person will share the secret even if the recipient already knows,
  and she will retain this knowledge (see the double arrows in the
  net).\footnote{Hence, in the Petri net semantics, we allow a transition
    to fire although the post-conditions is marked.}
  
  Here we use the stochastic semantics: only transitions that are
  enabled will be chosen (unless the marking is empty and no
  transition can fire).
  We assume that $p_T(d_2) = \sfrac{1}{3}$ and
  $p_T(d_1) = p_T(d_3) = p_T(d_4) = p_T(d_5) = \sfrac{1}{6}$, i.e.,
  user~$2$ is more talkative than the others.

  One of the states of the Markov chain is the marking $m=1100$
  ($K_1,K_2$ are marked -- users~$1$ and~$2$ know the secret -- and
  $K_3,K_4$ are unmarked -- users~$3$ and~$4$ do not). 
  In this situation transitions   $d_1,d_2,d_3$ are enabled.
  We normalize the probabilities and obtain that
  $d_2$ fires with probability $\sfrac{1}{2}$ and the other two with
  probability $\sfrac{1}{4}$. By firing $d_1$ or $d_2$ we stay in
  state $1100$, i.e., the corresponding Markov chain has a loop with
  probability $\sfrac{3}{4}$. Firing $d_3$ gives us a transition to
  state $1110$ (user~$3$ now knows the secret too) with
  probability $\sfrac{1}{4}$.

  The second net (Fig.~\ref{fig:pn-sir}) models the classical SIR
	  infection model \cite{keeling2005networks} \todo{\textbf{FSTTCS Rev1:} You cannot be certain
    this model is well known. You should provide explanations and/or a
    reference. Ba: done} for two persons. A person is
  \emph{susceptible} (represented by a token in place $S_i$) if he or
  she has not yet been infected. If the other person is infected
  (i.e. place $I_1$ or $I_2$ is marked), then he or she might also get
  \emph{infected} with the disease. Finally, people recover (or die),
  which means that they are \emph{removed} (places $R_i$). Again we
  use the stochastic semantics.

  The third net (Fig.~\ref{fig:pn-test}) models a test (for instance
  for an infection) that may have false positives and false
  negatives. A token in place $I$ means that the corresponding person
  is infected. Apart from $I$ there is another random variable $R$
  (for result) \todo{\textbf{FSCTTCS Rev1:} as T already denotes the
    set of transitions, it is not a good choice to use T as a random
    variable name. Ba: done} that tells whether the test is positive or
  negative. In order to faithfully model the test, we assign the
  following probabilities to the transitions:
  $p_T(\mathit{flp}) = P(R\mid \bar{I})$ (false or lucky positive:
  this transition can fire regardless of whether $I$ is marked, in
  which case the test went wrong and is only accidentally positive),
  $p_T(\mathit{inf}) = P(R\mid I) - P(R\mid \bar{I})$ (the remaining
  probability,\footnote{Here we require that
    $P(R\mid \bar{I}) \le P(R\mid I)$.} such that the probabilities of
  $\mathit{flp}$ and $\mathit{inf}$ add up to the true positive) and
  $p_T(\failop) = P(\bar{R}\mid I)$ (false negative). Here we use the
  independent semantics, assuming that we have a random test where the
  ground truth (infected or not infected) is independent of the
  firing probabilities of the transitions.
\end{example}

\section{Uncertainty Reasoning for Condition/Event Nets}
\label{sec:uncertainty-reasoning}

We now introduce the following scenario for uncertainty reasoning:
assume that we are given an initial probability distribution $p^0_*$
on the markings of the Petri net. We stipulate that the fail state $*$
cannot occur, \todo{\textbf{FSTTCS Rev1:} This is another sentence
  that adds ambiguity to the meaning of "failure". Overall, you should
  improve the way semantics of the net and failure during simulation
  are presented. Ba: done} assuming that the state of the net is
always some (potentially unknown) well-defined marking. If this fail
state would be reached in the Markov model, we assume that the marking
of the Petri net does not change, i.e., we perform a ``reset'' to the
previous marking.

Furthermore, we are aware of all firing probabilities of the various
transitions, given by the functions $(r_n)_{n\in\nat_0}$ and hence all
transition matrices $P^n$ that specify the transition probabilities at
step~$n$.

\todo[inline]{\textbf{FSTTCS Rev1:} The next paragraph adds even more ambiguity : "we observe the system and obtain sequences of success and failure". 
If you observe the running system, you can only record success, as the
system plays only transitions it can fire. Now if "observing the
system " means testing a random sequence of transitions, then indeed,
this sequence can contain a transition that is not firable. A this
seems to be the setting you consider, you should make it clear. Ba: done}
Then we observe the system and obtain a sequence of \emph{success} and
\emph{failure} occurrences. We are not told which exact transition
fires, but only if the firing is successful or fails (since the
pre-condition of the transition is not covered by the
marking). Note that according to our model, transitions \emph{can} be
chosen to fire, although they are not activated. This could happen if
either a user or the environment tries to fire such a transition,
unaware of the status of its pre-condition.
\todo{\textbf{CONCUR Rev1:} In l. 178, it is stated that the Models in the right 
plot of Fig. 6 cannot be computed by MBN but this contradicts the left plot 
where it is shown that MBN computes several models for values smaller than 25.\\
\textbf{Reb:} I cannot make sense of this comment...}
Failure corresponds to entering state $*$ and in this case
we assume the marking does not change.
That is, we keep the previous marking, but acquire additional
knowledge -- namely that firing fails -- which is used to update the
probability distribution according to
Prop.~\ref{prop:interpretation-prob} (by performing the corresponding
matrix multiplications, including normalization).

We use the following notation: let $M$ be a matrix indexed over
$\mathcal{M}\cup \{*\}$. Then we denote by $M_*$ the matrix obtained
by deleting the $*$-indexed row and column from $M$. Analogously for a
vector $p$. Note that $(M\cdot p)_* = M_*\cdot p_*$.
Furthermore if $p_*$ is a sub-probability vector, indexed over
$\mathcal{M}$, $\norm(p_*)$ stands for the corresponding
normalized vector, where the $m$-entry is
$p_*(m)/\big(\sum_{m'\in\mathcal{M}} p_*(m')\big)$. 

\begin{restatable}{prop}{PropInterpretationProb}
  \label{prop:interpretation-prob}
  Let $r_n\colon \mathcal{M}\times T_f\to[0,1]$ and
  $p^n\colon \mathcal{M}\cup \{*\}\to [0,1]$ be given as above.
  Let $N$ be a C/E net and let $(X_n)_{n\in{\nat_0}}$ be the Markov
  chain generated from $N, r_n$. Then
  \begin{itemize}
  \item
    $P(X_{n+1}=m' \mid X_{n+1}\neq *, X_n\neq *) = P(X_{n+1}=m' \mid
    X_{n+1}\neq *) = \norm(P^n_*\cdot p^n_*)(m')$
  \item
    $P(X_n=m \mid X_{n+1} = *, X_n\neq *) = \norm(F^n_*\cdot
    p^n_*)(m)$
  \end{itemize}
  where $p^n(m) = P(X_n=m)$, $p^n(*) = P(X_n=*)$ and $F^n$ is a
  diagonal matrix with
  $F^n(\bar{m}\mid \bar{m}) := P^n(*\mid \bar{m})$, $\bar{m}\in\mathcal{M}$, and
  $F^n(*\mid *) := P^n(*\mid *) = 1$, all other entries are $0$.
\end{restatable}


Hence, in case we observe a success we update the probability
distribution to $\bar{p}_{n+1}$ by computing $P^n_*\cdot \bar{p}^n$
(and normalizing). Instead, in the case of a failure we assume that
the marking stays unchanged, but by observing the failure we have
gathered additional knowledge, which means that we can replace
$\bar{p}_{n+1}$ by $F^n_*\cdot \bar{p}_n$ (after normalization).


$P^n_*$ and $F^n_*$ are typically not stochastic, but only
sub-stochastic.  For a (sub-)probability matrix $M_*$ and a
(sub-)probability vector $p_*$ it is easy to see that
$\norm(M_*\cdot p_*) = \norm(M_*\cdot \norm(p_*))$.
Hence another option is to omit the normalization steps and to
normalize at the very end of the sequence of
observations. Normalization may be undefined (in the case of the
$0$-vector), which signifies that we assumed an a priori probability
distribution that is inconsistent with reality.


\begin{example}
  \label{ex:uncertainty-reasoning}
  We get back to Ex.~\ref{ex:petri-nets} and discuss uncertainty
  reasoning. Assume that in the net in Fig.~\ref{fig:pn-sir}
  person~$1$ is susceptible ($S_1$ is marked), person~$2$ is infected
  ($I_2$ is marked) and the $i_j$-transitions have a higher rate
  (higher probability of firing) than the $r_j$-transitions. Then, in
  the next step the probability that both are infected is higher than
  the probability that $1$ is still susceptible and $2$ has recovered.

  Regarding the net in Fig.~\ref{fig:pn-test} we can show that in the
  next step, in the case of success, the probability distribution is
  updated in such a way that place~$I$ is marked with probability
  $P(I\mid R)$ and unmarked with probability $P(\bar{I}\mid R)$
  ($P(I\mid \bar{R})$, $P(\bar{I}\mid \bar{R})$ in the case of
  failure), exactly as required. For more details see
  \full{Appendix~\ref{sec:test-fpfn}}\short{\cite{bchk:uncertainty-reasoning-arxiv}}.
  \todo{\textbf{CONCUR Rev2:} This last paragraph is not very clear.}
\end{example}
 
\section{Modular Bayesian Networks}
\label{sec:bn-props}

\todo[inline]{\textbf{FSTTCS Rev1:} I am not a specialist of Bayesian networks, but I found the explanations of sections 4 to 6, and in particular the explanations for MBNs, Factors, and algo 11 rather technical. 
I think they need  more explanations. A part of the material in these sections is reused from a CONCUR paper in 2018 by the same authors ( [4] ) . I found the presentation in the 
CONCUR paper more accessible. As far as MBN are concerned, the difference between this paper and [4] is not really emphasized. The results in [4] show an improvement of MBN on 
the average runtime when performing 100 update operations, while this paper shows an improvement of MBNs use when addressing Petri net questions that require simulation. 
This paper explicitly eliminates variables in Bayesian networks while the technique to update in [4] merges nodes (which seems to be the same).}

\begin{wrapfigure}{r}{0.3\textwidth}
  \centering
  \scalebox{0.4}{\input{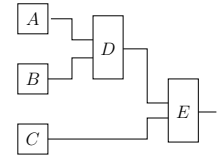}}
  \caption{An example Bayesian network}
  \label{fig:bn-variable-elimination}
\end{wrapfigure}
In order to implement the updates to the probability distributions
described above in an efficient way, we will now represent probability
distributions over markings symbolically as Bayesian networks
\cite{p:bayesian-networks,charniak1991bayesian}.  Bayesian networks
(BNs) model certain probabilistic dependencies of random variables
through conditional probability tables and a graphical representation.

Consider for instance the Bayesian network in
Fig.~\ref{fig:bn-variable-elimination}. Each node ($A$, $B$, $C$, $D$,
$E$) represents a binary random variable, where a node without
predecessors (e.g., $A$) is associated with the probabilities $P(A)$
and $P(\bar{A})$. Edges denote dependencies: for instance $D$ is
dependent on $A,B$, which means that $D$ is associated with a
conditional probability table (matrix) with entries $P(D\mid A,B)$,
similar for $E$ (entries of the form $P(E\mid D,C)$). In both cases,
the matrix contains $2\cdot 4 = 8$ entries.

We will later describe how to derive probability distributions and
marginal probabilities (for instance $P(E)$) from a Bayesian network.

We deviate from the literature on Bayesian networks in three respects:
first, since we will update and transform those networks, we need a
structure where we can easily express compositionality via sequential
and parallel composition. To this end we use the representation of
Bayesian networks via PROPs as in
\cite{f:causal-theories,jz:influence-bayesian}.  Second, we permit
sub-stochastic matrices. Third, we allow a node to have several
outgoing wires, \todo{\textbf{FSTTCS Rev1:} what is the interest of
  having several outgoing wires? Ba: done} whereas in classical
Bayesian networks a node is always associated to the distribution of a
single random variable. This is needed since we need to add nodes to a
network that represent stochastic matrices of arbitrary dimensions
(basically the matrices $P^n$ and $F^n$ of
Proposition~\ref{prop:interpretation-prob}). 
%
We rely on the notation introduced in
\cite{chhk:update-ce-nets-bayesian}, but extend it by taking the
last item above into account.


\heading{Causality Graphs}

The syntax of Bayesian networks is provided by \emph{causality graphs}
\cite{chhk:update-ce-nets-bayesian}.
For this we fix a set of node labels $G$, also called
\emph{generators}, where every $g\in G$ is associated with a type
$n_g\to m_g$, where $n_g,m_g\in\nat_0$.

\begin{defi}[Causality Graph (CG)]
  \label{def:CG}
  A \emph{causality graph (CG)} of type $n\to m$, $n,m\in\nat_0$, is a
  tuple $B = (V,\ell,s,\out)$ where
  \begin{itemize}
  \item $V$ is a set of \emph{nodes}
  \item $\ell\colon V\to G$ is a \emph{labelling function} that
    assigns a generator $\ell(v)\in G$ to each node \(v \in V\).
  \item $s\colon V\to W_B^*$ is the \emph{source function} that maps a
    node to a sequence of input wires, where $|s(v)| = n_{\ell(v)}$
    and
    $W_B = \{(v,p)\mid v \in V, p \in
    \{1,\dots,m_{\ell(v)}\}\}\cup\{i_1,\dots,i_n\}$ is the \emph{wire
      set}.
  \item $\out\colon \{o_1,\dots,o_m\}\to W_B$ is the \emph{output
      function} that assigns each output port to a wire.
  \end{itemize}
  Moreover, the corresponding directed graph (defined by $s$) has to
  be acyclic.
  
  We also define the \emph{target function} $t\colon V\to W_B^*$ with
  $t(v) = (v,1)\dots (v,m_{\ell(v)})$ and the set of \emph{internal wires}
  $\mathit{IW}_B = W_B\backslash
  \{i_1,\dots,i_n,\out(o_1),\dots,\out(o_m)\}$.
\end{defi}

We visualize such causality graphs by drawing the $n$ input wires on
the left and the $m$ outputs on the right. Each node $v$ is drawn as a
box, with $n_v$ ingoing wires and $m_v$ outgoing wires, ordered from
top to bottom. Connections induced by the source and by the output
function are drawn as undirected edges (see
Fig.~\ref{fig:bn-variable-elimination}). 

We define two operations on causality graphs: sequential composition
and tensor. Given $B$ of type $n\to k$ and $B'$ of type $k\to m$, the
sequential composition is obtained via concatenation, by identifying
the output wires of $B$ with the input wires of $B'$, resulting in
$B;B'$ of type $n\to m$. The tensor takes two causality graphs $B_i$
of type $n_i\to m_i$, $i\in\{1,2\}$ and takes their disjoint union,
concatenating the sequences of input and output wires, resulting in
$B_1\otimes B_2$ of type $n_1+n_2\to m_1+m_2$. For \full{a
  visualization see Fig.~\ref{fig:Kronecker-Multiplication} and for
}formal definitions see
\cite{chhk:update-ce-nets-bayesian,c:complex-networks-ppp}.


\heading{(Sub-)Stochastic Matrices}

The semantics of modular Bayesian networks is given by
\emph{(sub-)stochastic matrices}, i.e., matrices with entries from
$[0,1]$, where column sums will be at most $1$.  If the sum equals
exactly 1 we obtain stochastic matrices.

We consider only matrices whose dimensions are a power of two.
Analogously to causality graphs, we type matrices, and say that a
matrix has type $n\to m$ whenever it is of dimension $2^m \times 2^n$.
We again use a sequential composition operator $;$ that corresponds to
\emph{matrix multiplication} ($P;Q = Q\cdot P$) and the
\emph{Kronecker product} $\otimes$ as the tensor. More concretely,
given $P\colon n_1\to m_1$, $Q\colon n_2\to m_2$ we define
$P\otimes Q\colon n_1+n_2\to m_1+m_2$ as
$(P\otimes Q)(\bitvec{x}_1\bitvec{x}_2\mid \bitvec{y}_1\bitvec{y}_2) =
P(\bitvec{x}_1\mid \bitvec{y}_1)\cdot Q(\bitvec{x}_2\mid
\bitvec{y}_2)$ where $\bitvec{x}_i \in \{0,1\}^{m_i}$,
$\bitvec{y}_i \in \{0,1\}^{n_i}$.


\heading{Modular Bayesian Networks}

Finally, \emph{modular Bayesian networks}, adapted from
\cite{chhk:update-ce-nets-bayesian}, are causality graphs, where each
generator $g\in G$ is associated with a (sub-)stochastic matrix of
suitable type.

\begin{defi}[Modular Bayesian network (MBN)]
  \label{def:MBN}
  An MBN is a tuple $(B,\ev)$ where $B$ is a causality graph and $\ev$
  an \emph{evaluation function} that assigns to every generator
  $g\in G$ of type $n\to m$ a $2^m\times 2^n$ sub-stochastic matrix
  $\ev(g)$.  An MBN $(B,\ev)$ is called an \emph{ordinary Bayesian
    network (OBN)} whenever $B$ has no inputs (i.e. it has type
  $0 \to m$), each generator is of type $n\to 1$, $\out$ is a
  bijection and every node is associated with a stochastic matrix.
\end{defi}



We now describe how to evaluate an MBN to obtain a (sub-)stochastic
matrix. For OBNs -- which are exactly the Bayesian networks considered
in \cite{fgg:bayesian-network-classifiers} -- this coincides with the
standard interpretation and yields a probability vector of dimension
$m$.

\begin{defi}[MBN evaluation]
  \label{def:eval}
  Let $(B,\ev)$ be an MBN where $B$ is of type $n\to m$.
  Then $M_\ev(B)$ is a $2^m\times 2^n$-matrix, which is
  defined as follows:
  \[
    M_\ev(B)(\x_1\dots\x_m\mid \y_1\dots\y_n) = \sum_{b\in
      \mathcal{B}} \prod_{v \in V} \ev(l(v))\,(b(t(v))\mid b(s(v)))
  \]
  with $\x_1,\dots,\x_m,\y_1,\dots,\y_n\in\{0,1\}$.  $\mathcal{B}$ is
  the set of all functions $b: W_B \rightarrow \{0,1\}$ such that
  $b(i_j)=\y_j$, $b(out(o_k))=\x_k$, where $k\in\{1,\dots,m\}$,
  $j\in \{1,\dots,n\}$. The functions $b$ are applied pointwise to
  sequences of wires.
\end{defi}

Calculating the underlying probability distribution of an MBN can also
be done on a graphical level by treating every occurring wire as a
boolean variable that can be assigned either $0$ or $1$. Function
$b\in\mathcal{B}$ assigns the wires, ensuring consistency with the
input/output values. After the wire assignment, the corresponding
entries of each matrix $\ev(l(v))$ are multiplied. After iterating
over every possible wire assignment, the products are summed up.



Note that $M_\ev$ is compositional, it preserves sequential
composition and tensor. More formally, it is a functor between
symmetric monoidal categories, or -- more specifically -- between
CC-structured PROPs. \full{(For more details on PROPs see
  Appendix~\ref{sec:props}.)}\short{(More details on PROPs are given
  in the full version \cite{bchk:uncertainty-reasoning-arxiv}.)}

\begin{example}
  \label{ex:evaluate-bn}
  We illustrate Def.~\ref{def:eval} by evaluating the Bayesian network
  $(B',\ev)$ in Fig.~\ref{fig:bn-variable-elimination}.  This results
  in a $2\times 1$-matrix $M_\ev(B')$, assigning (sub-)probabilities
  to the only output wire in the diagram being $1$ or $0$,
  respectively.
  More concretely, we assign values to the four inner wires to obtain:
  \[ M_\ev(B')(\e) = \sum_{\a\in\{0,1\}} \sum_{\b\in\{0,1\}}
    \sum_{\c\in\{0,1\}} \sum_{\d\in\{0,1\}} \big(A(\a)\cdot B(\b)
    \cdot C(\c) \cdot D(\d\mid \a\b) \cdot E(\e\mid \c\d)\big), \]
  where $\a,\b,\c,\d,\e$ correspond to the output wire of the
  corresponding matrix ($A,B,C,D,E$).
\end{example}


\section{Updating Bayesian Networks}
\label{sec:updating-bn}


An MBN $B$ of type $0\to k$, as defined above, symbolically represents
a probability distribution on $\{0,1\}^k$, that is, a probability
distribution on markings of a net with $|S|=k$ places.

Under uncertainty reasoning
(cf. Section~\ref{sec:uncertainty-reasoning}), the probability
distribution in the next step $p^{n+1}$ is obtained by multiplying
$p^n$ with a matrix $M$ (either $P^n_*$ in the successful case or
$F^n_*$ in the case of failure).  \todo{\textbf{CONCUR Rev1:} The impact of a
  fail of a transition on the analysis is still unclear to me.}
Hence, a simple way to update $B$ would be to create an MBN $B_M$ with
a single node $v$ (labelled by a generator $g$ with $\ev(g) = M$),
connected to $k$ inputs and $k$ outputs.
Then the updated $B'$ is simply $B;B_M$ (remember that sequential
composition corresponds to matrix multiplication). However, at
dimension $2^k\times 2^k$ the matrix $M$ is huge and we would
sacrifice the desirable compact symbolic representation. Hence the aim
is to decompose $M = M'\otimes \Id$ where $\Id$ is an identity matrix
of suitable dimension. Due to the functoriality of MBN evaluation this
means composing with a smaller matrix and a number of identity wires
(see e.g. Fig.~\ref{fig:bn-gossip-2}).

\todo[inline]{\textbf{FSTTCS Rev1:} IMHO this decomposition is not straightforward, so these two paragraphs should explain better how knowing the flow relations among places and transitions can help 
you deriving an appropriate decomposition. My feeling is that in most nets, all places and transitions might be relevant according to distribution $p^n$ after a certain 
number of steps. If this is the case, decomposition-based techniques perform well only for a limited number of steps. \\
\textbf{FSTTCS Rev1:} First the connection between transitions 
and updates of the MBN are described rapidly on page 7. A reason to 
use MBNs is of course their modularity, and one can expect the concurrent 
structure of an MBN to facilitate a modular decomposition. However, 
this characteristics is addressed in only one sentence 
IMHO, this deserves more explanations. Ba: done}

This decomposition arises naturally from the structure of the Petri
net $N$, in particular if there are only relatively few transitions
that may fire in a step. In this case we intuitively have to attach a
stochastic matrix only to the wires representing the places connected
to those transitions, while the other wires can be left unchanged. If
there are several updates, we of course have to attach several
matrices, but each of them might be of a relatively modest size.

In order to have a uniform treatment of the various semantics, we
assume that for each step~$n$ there is a set $\bar{S}\subseteq S$ of
places\footnote{Without loss of generality we assume that the outputs
  have been permuted such that places in $\bar{S}$ occur first in the
  sequence of places.} and a set $\bar{T}\subseteq T_f$ of transitions
such that: (i) $r_n(m,t) = 0$ whenever $t\not\in\bar{T}$; (ii)
$r_n(m_1m_2,t) = \bar{r}(m_1,t)$ for some function $\bar{r}$ (where
$m_1$ is a marking of length $\ell = |\bar{S}|$, corresponding to the
places of $\bar{S}$); (iii) $\bar{S}$ contains at least
$\leftdot{t},\rightdot{t}$ for all $t\in\bar{T}$.  Intuitively,
$\bar{S}$, $\bar{T}$ specify the relevant places and transitions.

For the two Petri net semantics studied earlier, these conditions are
satisfied if we take as $\bar{T}$ the support of $p^n_T$ and as
$\bar{S}$ the union of all pre- and post-sets of $\bar{T}$. The
function $r_n$ can in both cases be defined in terms of $\bar{r}$: in
the independent case this is obvious, whereas in the stochastic net
case we observe that $r_n(m,t)$ is only dependent on $p^n_T$ and on the
set of transitions that is enabled in $m$ and this can be derived from
$m_1$.

Now, under these assumptions, we can prove that we obtain the
decomposition mentioned above.


\begin{restatable}{prop}{PropPFMatrices}
  \label{prop:p-f-matrices}
  Assume that $N$ is a condition/even-net together with a function
  $r_n$. Assume that we have $\bar{S}\subseteq S$,
  $\bar{T}\subseteq T_f$ satisfying the conditions above. Then
  \begin{itemize}
  \item $P^n_* = P'\otimes \Id_{2^{k-\ell}}$ where
    $P'(m'_1\mid m_1) = \sum_{t\in\bar{T},m_1\overset{t}{\Rightarrow} m'_1}
    \bar{r}(m_1,t)$.
  \item $F^n_* = F'\otimes \Id_{2^{k-\ell}}$ where
    $F'(m'_1\mid m_1) =
    \sum_{t\in \bar{T},m_1\not\overset{t}{\Rightarrow}}
    \bar{r}(m_1,t)$ if $m_1=m'_1$ and $0$ otherwise.
  \end{itemize}
  Here $P',F'$ are $2^\ell\times 2^\ell$-matrices and $m_1,m'_1\subseteq\bar{S}$. Note also that we
  implicitly restricted the firing relation to the markings on
  $\bar{S}$.
\end{restatable}
\todo{\textbf{FSTTCS Rev1:} Why is this proposition interesting ? Can
  we save simulation time via computation of a tensor product on
  smaller matrices ( and what is the gain wrt MBN approach ?) Ba: done}

\begin{example}
  \label{ex:update-bn}
  \todo{\textbf{CONCUR Rev2:} It is good that Section 3 contains an explanation of 
  the setting of uncertainty reasoning for Petri nets. However, when getting 
  back to this (after a lot of other definitions) in examples, a reminder or 
  reference would be helpful (as in Example 4, where the initial assumption 
  seemed to be new).}
  In order to illustrate this, we go back to gossip diffusion
  (Fig.~\ref{fig:pn-gossip}, Ex.~\ref{ex:petri-nets}).
  Our input is the following: an initial probability distribution,
  describing the a priori knowledge, given by an MBN. Here we have no
  information about who knows or does not know the secret and hence we
  assume a uniform probability distribution over all markings.  This
  is represented by the Bayesian network in Fig.~\ref{fig:bn-gossip-1}
  where each node is associated with a $2\times 1$-matrix (vector)
  $K_i$ where both entries are $\sfrac{1}{2}$.

  Also part of the input is the family of transition distributions
  $(r_n)_{n\in\nat_0}$. Here we assume that the firing probabilities
  of transitions are as in Example~\ref{ex:petri-nets}, but not all
  users are active at the same time. We have information that in the
  first step only users~$1$ and~$2$ are active,\todo{\textbf{CONCUR Rev2:} Is
    the assumption of knowing which users are active (= an initial
    marking) generally given? Or is this just an example here, to
    compute the probabilities for the transitions and the two
    corresponding matrix entries? I.e., do you need to consider all
    possible initial markings (and hence the whole matrix) to compute
    the resulting probability for K\_3?  It’s not very clear, how far
    the assumption goes.}  hence by normalization we obtain
  probabilities $\sfrac{1}{4}$, $\sfrac{1}{2}$, $\sfrac{1}{4}$ for
  transitions $d_1$, $d_2$, $d_3$ (the other transitions are
  deactivated).

  Now we observe a success step. According to
  Sct.~\ref{sec:uncertainty-reasoning} we can make an update with
  $P_*$ where $P$ is the transition matrix of the Markov chain. Since
  none of the transitions is attached to place $K_4$ the optimizations
  of this section allow us to represent $P_*$ as $P'\otimes \Id_2$
  where $P'$ is an $8\times 8$-matrix.  E.g., as discussed in
  Ex.~\ref{ex:petri-nets}, we have $P'(110\mid 110) = \sfrac{3}{4}$,
  $P'(111\mid 110) = \sfrac{1}{4}$.  This matrix is simply attached to
  the modular Bayesian network (see Fig.~\ref{fig:bn-gossip-2}).

  Now assume that it is our task to compute the probability that place
  $K_3$ is marked. For this, we compute the corresponding marginal
  probabilities by terminating each output wire (apart from the third
  one) (see Fig.~\ref{fig:bn-gossip-3}). ``Terminating a wire'' means
  to remove it from the output wires. This results in summing up over
  all possible values assigned to each wire, where we can completely
  omit the last component, which is the unit of the Kronecker product.
  Note that the resulting vector is sub-stochastic and still has to be
  normalized. The normalization factor can be obtained by terminating
  also the remaining third wire,\todo{\textbf{FSTTCS Rev1:} what does
    "closing a wire" mean?  Ba: done, Be: ist closing=terminating?}
  which gives us the probability mass of the sub-probability
  distribution. Our implementation will now tell us that place $K_3$
  is marked with probability $\sfrac{5}{8}$.\todo{\textbf{Ba:} do the
    calculation in the appendix?}\todo{\textbf{CONCUR Rev2:} (last
    paragraph; "Now assume..")  It is not very clear how this
    works. It would be helpful to refer to the relevant equations and
    to actually state the computation.}
\end{example}
\todo[inline]{\textbf{CONCUR Rev2:} it is hard to distinguish the assumptions 
which need to be taken for the whole approach to work from the choice 
of values to compute exemplary values. Also, the description of the 
computation (of the probability for K\_3 to be marked) is very vague 
and hard to follow.}

\begin{figure}
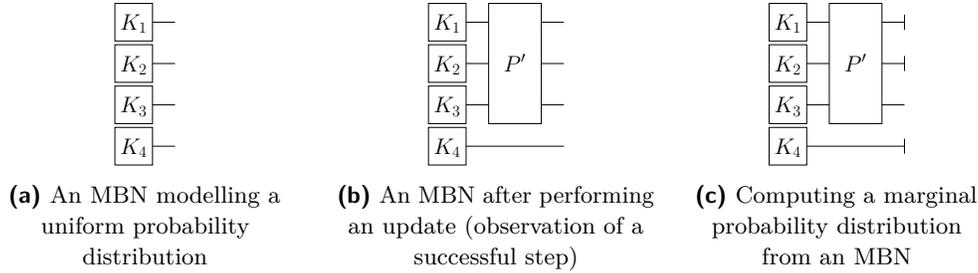

  \captionsetup[subfigure]{justification=centering}
  \begin{subfigure}[t]{0.3\textwidth}
    \centering \scalebox{0.5}{\input{bn-gossip-1-x.tex}}
    \caption{An MBN modelling a uniform probability distribution}
    \label{fig:bn-gossip-1}
  \end{subfigure}
  \quad
  \begin{subfigure}[t]{0.3\textwidth}
    \centering
    \scalebox{0.5}{\input{bn-gossip-2-x.tex}}
    \caption{An MBN after performing an update (observation of a
      successful step)}
    \label{fig:bn-gossip-2}
  \end{subfigure}
  \quad
  \begin{subfigure}[t]{0.28\textwidth}
    \centering
    \scalebox{0.5}{\input{bn-gossip-3-x.tex}}
    \caption{Computing a marginal probability distribution from an
      MBN}
    \label{fig:bn-gossip-3}
  \end{subfigure}
  \caption{Example: transformation of modular Bayesian networks}
\end{figure}

\section{Variable Elimination and Tree Decompositions}
\label{sec:variable-elimination}

 \heading{Motivation}

 Given a modular Bayesian network, it is inefficient to obtain the
 full distribution, not just from the point of view of the
 computation, but also since its direct representation is of
 exponential size. \todo{\textbf{FSTTCS Rev1:} computing the full
   distribution is time consuming and often not needed? Ba:
   done}However what we often need is to compute a marginal
 distribution (e.g., the probability that a certain place is marked)
 or a normalization factor for a sub-stochastic probability
 distribution (cf. Ex.~\ref{ex:update-bn}). Another application would
 be to transform an MBN into an OBN, by isolating that part of the
 network that does not conform to the properties of an OBN, evaluating
 it and replacing it by an equivalent OBN.

Def.~\ref{def:eval} gives a recipe for the evaluation, which is
however quite inefficient. Hence we will now explain and adapt the
well-known concept of variable elimination 
\cite{d:bucket-elimination,d:modeling-reasoning-bn}.
Let us study the problem with a concrete example. Consider the
Bayesian network $B'$ in Fig.~\ref{fig:bn-variable-elimination} and
its evaluation described in Ex.~\ref{ex:evaluate-bn}.  If we perform
this computation one has to enumerate $2^4 = 16$ bit vectors of length
$4$. Furthermore, after eliminating $\d$ we have to represent a matrix
(also called \emph{factor} in the literature on Bayesian networks)
that is dependent on four random variables ($\a,\b,\c,\e$), hence we
say that it has width~$4$ ($2^4 = 16$ entries).

However, it is not difficult to see that we can -- via the
distributive law -- reorder the products and sums to obtain a more
efficient way of computing the values:
\[ M_\ev(B')(\e) = \sum_{\d\in\{0,1\}} \big(\sum_{\c\in\{0,1\}}
  \big(\sum_{\b\in\{0,1\}} \big(\sum_{\a\in\{0,1\}} \big(A(\a) \cdot
  D(\d\mid \a\b)\big)\cdot B(\b)\big) \cdot C(\c)\big) \cdot E(\e\mid
  \c\d)\big). \] In this way we obtain smaller matrices,
the largest matrix (or factor) that occurs is $D$ (width~$3$).
Choosing a different elimination order might have been worse. For
instance, if we had eliminated $\d$ first, we would have to deal with a
matrix dependent on $\a,\b,\c,\e$ (width~$4$).

\heading{Variable elimination}

The literature of Bayesian networks
\cite{d:bucket-elimination,d:modeling-reasoning-bn} extensively
studies the best variable elimination order and discusses the
relation to treewidth. For our setting we have to extend the results in the
literature, since we also allow generators with more than one output.


\begin{defi}[Elimination order]
  \label{def:elimination-order}
  Let $B = (V,\ell,s,\out)$ be the causality graph of a modular
  Bayesian network of type $n\to m$. As in Def.~\ref{def:CG} let $W_B$
  be the set of wires.  \todo{\textbf{CONCUR Rev2:} Probably the other
    reoccurring variables also refer to Definition 5?}

  We define an undirected graph $U_0$ that has as vertices\footnote{We
    talk about the \emph{nodes} of an MBN $B$ and the
    \emph{vertices} of an undirected graph $U_i$.} the wires $W_B$ and
  two wires $w_1,w_2$ are connected by an edge whenever they are
  connected to the same node.  More precisely, they are connected
  whenever they are input or output wires for the same node
  (i.e. $w_1,w_2$ are both in $s(v)t(v)$ for a node $v\in V$).

  Now let $w_1,\dots,w_k$ (where $k = |\mathit{IW}_B|$) be an ordering of
  the internal wires, a so-called \emph{elimination ordering}. We
  update the graph $U_{i-1}$ to $U_i$ by removing the next wire $w_i$
  and connecting all of its neighbours by edges (so-called \emph{fill
    in}). External wires are never eliminated.
  The \emph{width} of the elimination ordering is the size of the
  largest clique
  \todo{\textbf{CONCUR Rev2:} What is a clique / how is it defined? Is it a 
  strongly connected component?}
  that occurs in some graph $U_i$. The
  \emph{elimination width} of $B$ is the least width taken over all
  orderings. 
\end{defi}


In the case of Bayesian networks, the set of wires of an OBN
corresponds to the set of random variables. In the literature, the
graph $U_0$ is called the moralisation of the Bayesian network, it is
obtained by taking the Bayesian network (an acyclic graph), forgetting
about the direction of the edges, and connecting all the parents
(i.e., the predecessors) of a random variable, i.e. making them form a
clique. This results in the same graph as the construction described
above.

To introduce the algorithm, we need the notion of a \emph{factor},
already hinted at earlier.

\begin{defi}[Factor]
  Let $(B,\ev)$ be a modular Bayesian network with a set of wires
  $W_B$.  A \emph{factor} $(f,\tilde{w})$ of size $s$ consists of a
  map $f\colon \{0,1\}^s\to [0,1]$ together with a sequence of wires
  $\tilde{w}\in W_B^*$.
  \todo{\textbf{CONCUR Rev2:} Here, W\_B* appears to be the set of sequences 
  of elements from W\_B. Does this match the interpretation from 
  Section 4 (and explain the sequential application of b?)? A clear 
  definition of W\_B* at the first occurrence would have been helpful.}
  We require that $\tilde{w}$ is of length $s$
  ($|\tilde{w}| = s$) and does not contain duplicates.

  Given a wire $w\in W_B$ and a multiset $\mathcal{F}$ of factors, we
  denote by $C_w(\mathcal{F})$ all those factors
  $(f,\tilde{w})\in \mathcal{F}$ where $\tilde{w}$ contains $w$. By
  $X_w(\mathcal{F})$ we denote the set of all wires that occur in
  the factors in $C_w(\mathcal{F})$, apart from $w$.

\end{defi}


We now consider an algorithm that computes the probability
distribution represented by a modular Bayesian network of type
$n\to m$. We assume that an evaluation map $\ev$, mapping generators
to their corresponding matrices, and an elimination order
$w_1,\dots,w_k$ of internal wires is given. Furthermore, given a
sequence of wires $\tilde{w} = w'_1\dots w'_s$ and a bitstring
$\X = \x_1\dots \x_s$, we define the substitution function
$b_{\tilde{w},\X}$ from wires to bits as
$b_{\tilde{w},\X}(w'_j) = \x_j$.


\begin{algo}[Variable elimination]~
  \label{alg:variable-elimination}
  \todo{\textbf{FSTTCS Rev1:} Algorithm 11 is hard to follow and deserves more explanations (and possibly an example) to show how it is used to eliminate a variable. }

  \noindent\emph{Input:} An MBN $(B,\ev)$ of type $n\to m$
  \begin{itemize}
  \item Let $\mathcal{F}_0$ be the initial multiset of
    factors. For each node $v$ of type $n_{\ell(v)}\to m_{\ell(v)}$,
    it contains the matrix $\ev(v)$, represented as a factor $f$,
    together with the sequence $s(v)t(v)$. That is
    $f(\X\Y) = \ev(v)(\Y\mid \X)$ where $\X\in \{0,1\}^{n_{\ell(v)}}$,
    $\Y\in \{0,1\}^{m_{\ell(v)}}$.
  
  \item Now assume that we have a set $\mathcal{F}_{i-1}$ of factors
    and take the next wire $w_i$ in the elimination order. We choose
    all those factors that contain $w_i$ and compute a new
    factor~$(f,\tilde{w})$.
    Let $\tilde{w}$ be a sequence that contains all wires of
    $X_w(\mathcal{F}_{i-1})$ (in arbitrary order, but without
    duplicates). Let $s = |\tilde{w}|$. Then $f$ is a function of type
    $f\colon \{0,1\}^s \to [0,1]$, defined as:
    \[ f(\Y) = \sum_{\z\in \{0,1\}} \prod_{(g,\tilde{w}^g)\in
        C_{w_i}(\mathcal{F}_{i-1})}
      g(b_{\tilde{w}w_i,\Y\z}(\tilde{w}^g)).  \] We set
    $\mathcal{F}_i = \mathcal{F}_{i-1}\backslash
    C_{w_i}(\mathcal{F}_{i-1})\cup \{(f,\tilde{w})\}$.
  \item After the elimination of all wires we obtain a multiset of
    factors $\mathcal{F}_k$, whose sequences contain only input and
    output wires.
    The resulting probability distribution is
    $p\colon \{0,1\}^{n+m}\to [0,1]$, where $\X\in\{0,1\}^n$,
    $\Y\in\{0,1\}^m$, $\tilde{\iota} = i_1\dots i_n$,
    $\tilde{o} = \out(o_1)\dots \out(o_m)$:
    \[ p(\X\Y) = \prod_{(f,\tilde{w}^f)\in \mathcal{F}_k}
      f(b_{\tilde{\iota}\tilde{o},\X\Y}(\tilde{w}^f))
    \]
  \end{itemize}

\end{algo}

That is, given the next wire $w_i$ we choose all factors that contain
this wire, remove them from $\mathcal{F}_{i-1}$ and multiply them,
while eliminating the wire. The next set is obtained by adding the
new factor.
Finally, we have factors that contain only input and output wires and
we obtain the final probability distribution by multiplying them.


\begin{restatable}{prop}{PropVariableElimination}
  \label{prop:variable-elimination}
  Given a modular Bayesian network $(B,\ev)$ where $B$ is of type $n\to m$,
  Algorithm~\ref{alg:variable-elimination} computes its corresponding
  (sub-)stochastic matrix $M_\ev(B)$, that is
  \[ M_\ev(B)(\Y\mid \X) = p(\X\Y) \qquad \mbox{for $\X\in\{0,1\}^n$,
      $\Y\in\{0,1\}^m$}. \] Furthermore, the size of the largest
  factor in any multiset $\mathcal{F}_i$ is bounded by the width of
  the elimination ordering.
\end{restatable}

\heading{Comparison to Treewidth}

We conclude this section by investigating the relation between
elimination width and the well-known notion of treewidth
\cite{bk:treewidth-computations-I}.



\todo{\textbf{FSTTCS Rev1:} Propositions 12 and 13 are rather expected, as one expects a particular ordering for elimination to be as efficient as possible, and at least more efficient that the worst ordering. 
These propositions are welcome, of course, because they prove that the proposed approach can only improve varaible elimination. However, they are not commented. 
Though these propositions are  desirable properties of MBN (regardless of whether these MBN originate from a Petri net or not), I think they are not a major result of the paper.
IMHO, the main result of the paper is the demonstration of the efficiency of MBNs through experiments.}

%
%
%

\begin{defi}[Treewidth of a causality graph]
  \label{def:treewidth-bn}
  Let $B = (V,\ell,s,\out)$ be a causality graph of type $n\to
  m$. A \emph{tree decomposition} for $B$ is an undirected tree
  $T = (V_T,E_T)$ such that
  \begin{itemize}
  \item every node $t\in V_T$ is associated with a bag $X_t\subseteq
    W_B$, 
  \item every wire $w\in W_B$ in contained in at least one bag $X_t$,
  \item for every node $v\in V$ there exists a bag $X_t$ such that all
    input and output wires of $v$ are contained in $X_t$ (i.e., all
    wires in $s(v)$ and $t(v)$ are in $X_t$)
    \emph{and}
  \item for every wire $w\in W_B$, the tree nodes $\{t\in V_T \mid
    w\in X_t\}$ form a subtree of $T$.
  \end{itemize}
  The width of a tree decomposition is given by $\max_{t\in V_T} |X_t|
  - 1$.

  The treewidth of $B$ is the minimal width, taken over all tree
  decompositions.
\end{defi}

Note that the treewidth of a causality graph corresponds to the
treewidth of the graph $U_0$ from Def.~\ref{def:elimination-order}. Now
we are ready to compare elimination width and treewidth.

\begin{restatable}{prop}{PropTreewidthBN}
  \label{prop:treewidth-bn}
  Elimination width is always an upper bound for treewidth and they
  coincide when $B$ is a causality graph of type $0\to 0$. 
  For a network of type $0\to m$ the treewidth may be strictly smaller
  than the elimination width.
\end{restatable}

The treewidth might be strictly smaller since we are now allowed to
eliminate output wires.  However, it is easy to see that the treewidth
plus the number of output wires always provides an upper bound for the
elimination width.

The paper \cite{bk:treewidth-computations-I} also discusses heuristics
for computing good elimination orderings, an opimization problem that
is $\mathsf{NP}$-hard.
Hence the treewidth of a causality graph gives us an upper bound for
the most costly step in computing its corresponding probability
distribution.  \cite{kbg:necessity-treewidth-bayesian} shows that a
small treewidth is actually a necessary condition for obtaining
efficient inference algorithms.


We can also compare elimination width to the related notion of term
width, more details can be found in
\full{Appendix~\ref{sec:termwidth}}\short{\cite{bchk:uncertainty-reasoning-arxiv}}.

\section{Implementation and Runtime Results}
\label{sec:implementation}
\todo[inline]{\textbf{FSTTCS Rev1:} Now, these results need further discussion: firts, 
they are average results, achieved after a bounded number of steps. On may wonder what occurs in the worst cases. In fact, if a Petri net is live, after a few steps one can expect 
in general any place to have a strictly positive probability to be marked, and similarly one can expect all transitions to be firable with probability >0. So the hypothesis of few 
relevant transitions and places holds for a limited number of steps. So it is hard to figure whether MBN bring a general improvement to Stochastic net analysis or a 
temporary advantage for some nets on a limited number of simulation steps. }
\todo[inline]{\textbf{FSTTCS Rev2:} One question for the authors here is whether they have implemented the SIR model, and other practical examples as well and if so what is their performance results on them. Do they satisfy the low connectivity assumptions -- in other words can the motivation be validated by the implementation?}
We extended the implementation presented in the predecessor paper
\cite{chhk:update-ce-nets-bayesian} by incorporating probabilistic
Petri nets and elimination orderings, in order to evaluate the
performance of the proposed concepts.  The implementation is open
source and freely available from
GitHub.\footnote{\url{https://github.com/RebeccaBe/Bayesian-II}}

Runtime results were obtained by randomly generating Petri nets with
different parameters, e.g. number of places, transitions and tokens,
initial marking. The maximal number of places in pre- and
post-conditions is restricted to three and at most
five transitions are enabled in each step.
With these parameters, the worst case scenario is the creation of a
matrix of type $30 \rightarrow 30$.  After the initialization of a
Petri net, which can be interpreted with either semantics
(independent/stochastic), transitions and their probabilities are
picked at random. Then we observe either success or failure and update
the probability distribution accordingly.

We select the elimination order via a heuristics by preferring wires
with minimal degree in the graph $U_i$
(cf. Def~\ref{def:elimination-order}).
Furthermore we apply a few optimizations: Nodes with no output wires
will be evaluated first, nodes without inputs second.  The observation
of a failure will generate a diagonal matrix, which enables an
optimized evaluation, as its input and output wires have to carry the
same value (otherwise we obtain a factor $0$).  In addition, we use
optimizations whenever we have definitive knowledge about the marking
of a particular place (of a pre-condition), by drawing conclusions
about the ability to fire certain transitions.

The plot on the left of Fig.~\ref{fig:runtime_results}
compares runtimes when incorporating ten
success/failure observations directly on the joint distribution
(i.e. the naive representation of a probability distribution)
versus our MBN implementation. We initially assume a uniform distribution of
tokens and calculate the probability that the first place is marked
after the observations. Both approaches evaluate the same Petri net and
therefore calculate the same results. The data is for the independent
semantics, but it is very similar for the stochastic semantics.

While the runtime increases exponentially when using joint
distributions, our MBN implementation stays relatively constant (see
Fig.~\ref{fig:runtime_results}, left). Due to memory issues, handling
Petri nets with more than 30 places is not anymore feasible for the
direct computation of joint distributions.
%
%
We use the median for comparison (see Fig.~\ref{fig:runtime_results},
left), but if an MBN consists of very large matrices, the evaluation
time will be rather high. The right plot of
Fig.~\ref{fig:runtime_results} shows this correlation, where colours
denote the runtime and the $y$-axis represents the number of wires
attached to the largest matrix. (Here we actually count equivalence
classes by grouping those wires that have to carry the same value, due
to their attachment to a diagonal matrix, see also the optimization
explained above.)

\begin{figure}
  \includegraphics[width=.476357146\textwidth]{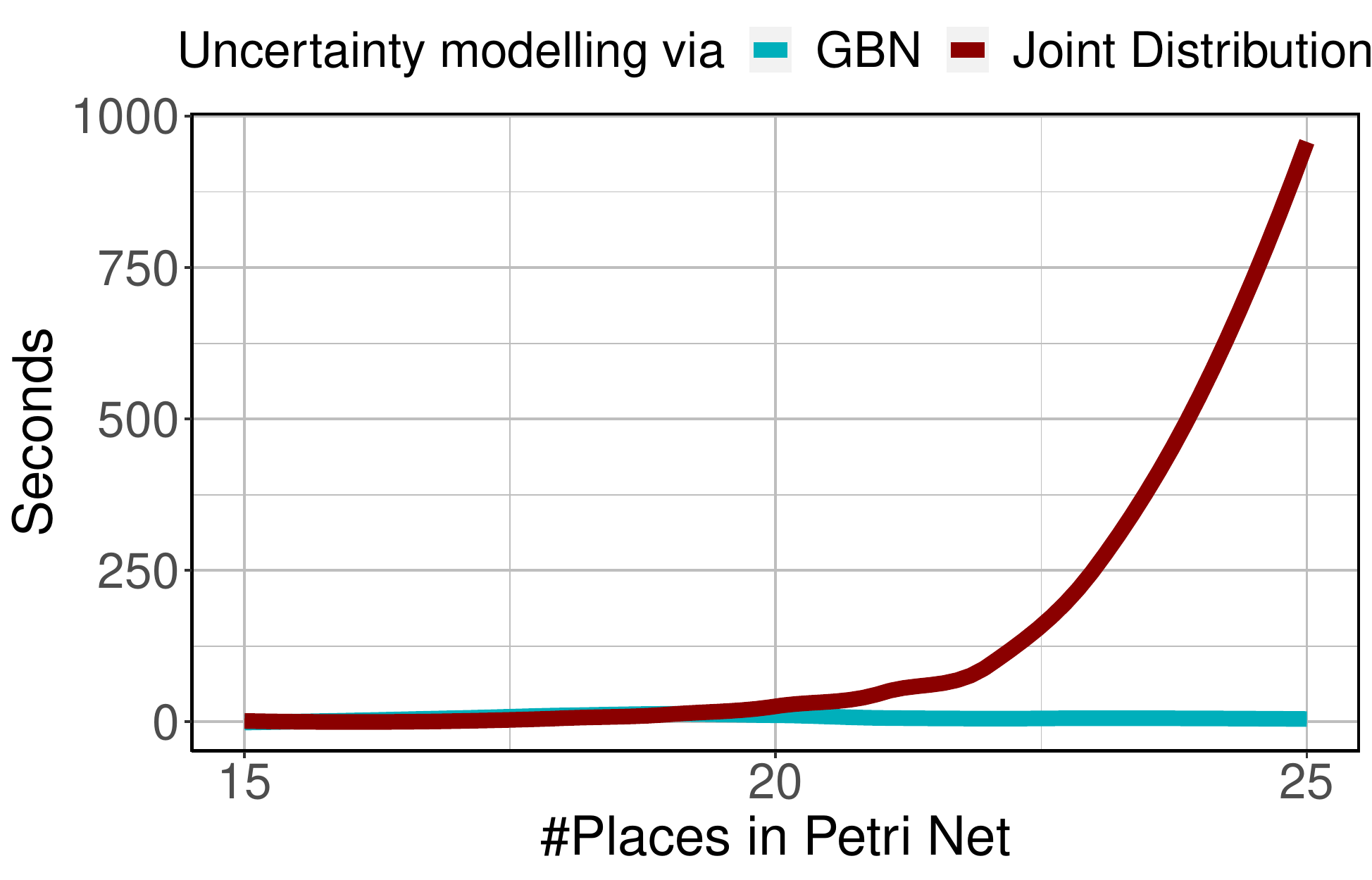}
  \hspace{0.2em}
  \includegraphics[width=.500785718\textwidth]{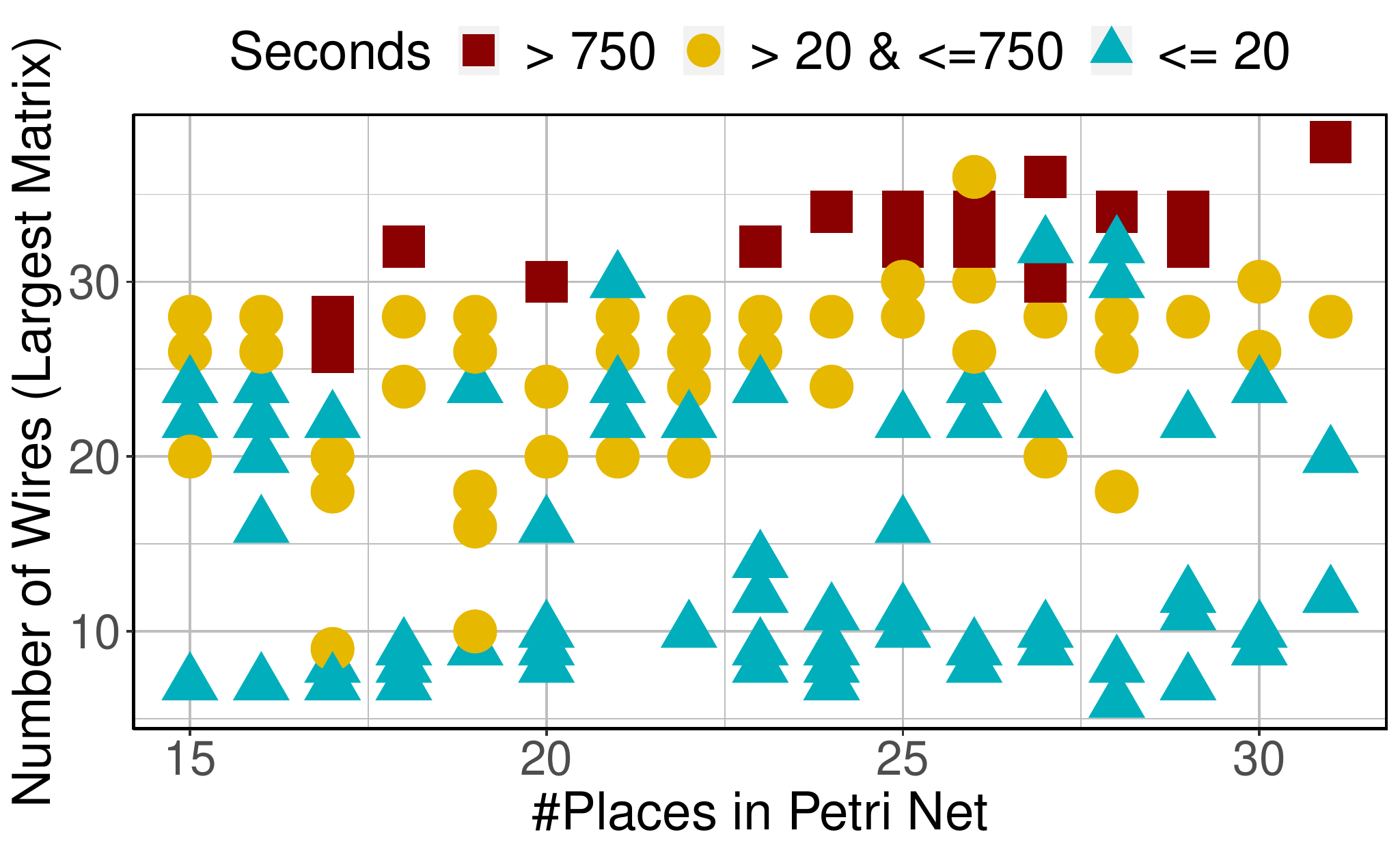}
  
  
  \caption{Left: Median of runtimes performing after 10 transitions on
    a Petri net. Right: Effect of large matrices on the runtimes of
    the MBN implementation.}
  \label{fig:runtime_results}

\end{figure}
\todo{\textbf{FSTTCS Rev1:} Figures not readable in b/w}

The advantage of our approach decreases when we have substantially
more places in the pre- and post-set, more transitions that may fire
and a larger number of steps, since then the Bayesian network is more
densely connected and contains larger matrices.  Furthermore, one
might generally expect the state (containing tokens or not) of places
of the Petri net to become more and more coupled over time, as more
transitions have fired, decreasing the performance improvement we gain
from using MBNs.  However, recall that the transitions that can fire
at any time are explicitly controlled by the input $p_T^n$. This
allows our model to capture situations where different parts of the
network stay uncoupled over time and where using MBNs is an advantage.
Furthermore the observation of a failure allows an optimizated
variable elimination, as explained above.


  

  

  


  

\section{Conclusion}
\label{sec:conclusion}

\todo[]{\textbf{FSTTCS Rev2:} One immediate question!
here is how this approach relates to usual belief-propagation methods that are used for efficient marginal computations on Bayesian networks? Can they be applied here and if not, where do they fail? }

We propose a framework for uncertainty reasoning for probabilistic
Petri nets that represents probability distributions compactly via
Bayesian networks. In particular we describe how to efficiently update
and evaluate Bayesian networks.

\smallskip

\noindent\emph{Related work:} Naturally, uncertainty reasoning has been
considered in many different scenarios (for an overview see
\cite{h:reasoning-uncertainty}). Here we review only those approaches
that are closest to our work.

In \cite{chhk:update-ce-nets-bayesian} we studied a simpler scenario
for nets whose transitions do not fire probabilistically, but are
picked by the observer, resulting in a restricted set of update
operations.  Rather than computing marginal distributions directly via
variable elimination as in this paper, our aim there was to transform
the resulting modular Bayesian network into an ordinary one. Since the
updates to the net were of a simpler nature, we were able to perform
this conversion. Here we are dealing with more complex updates where
this can not be done efficiently. Instead we are concentrating on
extracting information, such as marginal distributions, from a
Bayesian network.

Furthermore, uncertainty reasoning as described in
Sct.~\ref{sec:uncertainty-reasoning} is related to the methods used
for hidden Markov models \cite{r:hidden-markov-models}, where the
observations refer to the states, whereas we (partially) observe the
transitions.

There are several proposals which enrich Petri nets with a notion of
uncertainty: possibilistic Petri nets
\cite{llc:possibilistic-petri-nets}, plausible Petri nets
\cite{ccpa:plausible-petri-nets} that combine discrete and continuous
processes or fuzzy Petri nets
\cite{cvd:fuzzy-pn,s:generalized-fuzzy-pn} where firing of transitions
is governed by the truth values of statements. Uncertainty in
connection with Petri nets is also treated in
\cite{kb:uncertainty-pn,jr:uncertainty-initial-petri}, but without
introducing a formal model. As far as we know neither approach
considers symbolic representation of probability distributions via
Bayesian networks.

In \cite{bmm:unifying-bn-pn} the authors exploit the fact that Petri
nets also have a monoidal structure and describe how to convert an
occurrence (Petri) net with a truly concurrent semantics into a
Bayesian network, allowing to derive probabilistic information, for
instance on whether a place will eventually be marked. This is
different from our task, but it will be interesting to compare further
by unfolding our nets and equipping them with a truly concurrent
semantics, based on the probabilistic information from the
time-inhomogeneous Markov chain.

We instead propose to use Bayesian networks as symbolic
representations of probability distributions. An alternative would be
to employ multi-valued (or multi-terminal) binary decision diagrams
(BDDs) as in \cite{hms:multi-bdds-ctmcs}. An exact comparison of both
methods is left for future work. We believe that multi-valued BDDs
will fare better if there are only few different numerical values in
the distribution, otherwise Bayesian networks should have an
advantage.

As mentioned earlier, representing Bayesian networks by PROPs or
string diagrams is a well-known concept, see for instance
\cite{f:causal-theories,jz:influence-bayesian}. The paper
\cite{jkz:causal-string-diagram} describes another transformation of
Bayesian networks by string diagram surgery that models the effect of
an intervention.

In addition there is a notion of dynamic Bayesian networks
\cite{m:dynamic-bayesian-networks}, where a random variable has a
separate instance for each time slice. We instead keep only one
instance of every random variable, but update the Bayesian network
itself.

In addition to variable elimination, a popular method to compute marginals of a probability distribution is based on belief propagation and junction trees \cite{DBLP:junction-tree-algorithm}. In order to assess the potential efficiency gain, this approach has to be adapted for modular Bayesian networks. However due to the dense interconnection and large matrices of MBNs, an improvement in runtime is unclear and deserves future investigation.

  

  





\smallskip

\noindent\emph{Future work:}
One interesting avenue of future work is to enrich our model with
timing information by considering continuous-time Markov chains
\cite{t:introduction-markov-chains}, where firing delays are sampled
from an exponential distribution. Instead of asking about the
probability distribution after $n$~steps we could instead ask about
the probability distribution at time~$t$.

We would also like to add mechanisms for controlling the system, such
as transitions that are under the control of the observer and can be
fired whenever enabled. Then the task of the observer would be to
control the system and guide it into a desirable state. In this vein
we are also interested in studying stochastic games
\cite{c:complexity-stochastic-game} with uncertainty.

The interaction between the structure of the Petri net and the
efficiency of the analysis method also deserves further study. For
instance, are free-choice nets \cite{de:free-choice-petri} -- with
restricted conflicts of transitions -- more amenable to this type of
analysis than arbitrary nets?

Recently there has been a lot of interest in modelling compositional
systems via string diagrams, in the categorical setting of symmetric
monoidal categories or PROPs \cite{ck:picturing-quantum}. In this
context it would be interesting to see how the established notion of
treewidth \cite{bk:treewidth-computations-I} and its algebraic
characterizations \cite{ce:graph-structure-mso} translates into a notion of
width for string diagrams. We started to study this for the notion of
term width, but we are not aware of other approaches, apart from
\cite{cs:compositional-graph-theory} which considers monoidal width.







\bibliographystyle{plain}
\bibliography{references}

\full{

\appendix

\section{PROPs}
\label{sec:props}

Both, causality graphs and (sub-)stochastic matrices, can be seen in
the context of PROPs \cite{m:categorical-algebra} (where PROP stands
for ``products and permutations category''), a categorical notion that
formalizes string diagrams.

Since we do not need the full theory behind PROPs to obtain our
results and because of space restrictions, we did not define PROPs
explicitly within the main part of the paper.

However, here we formally introduce the mathematical structure
underlying modular Bayesian networks: CC-structured PROPs, i.e., PROPs
with commutative comonoid structure, a type of strict symmetric
monoidal category.
 
A CC-structured PROP is a symmetric monoidal category whose objects
are natural numbers and arrows are terms. In particular every term $t$
has a type $n \rightarrow m$ with $n, m \in \nat_0$.


There are two operators that can be used to combine terms: sequential
(;) and parallel ($\otimes$) composition, also called tensor (see
Fig~\ref{fig:Kronecker-Multiplication}). Sequential composition
corresponds to the categorical composition and combines two terms
$t_1: n \rightarrow l$ and $t_2\colon l \rightarrow m$ to a term
$t_1 ; t_2 = t\colon n \rightarrow m$. When applied to
$t_1: n_1 \rightarrow m_1$ and $t_2\colon n_2 \rightarrow m_2$, the
tensor operator produces a term
$t_1 \otimes t_2 = t\colon n_1+n_2 \rightarrow m_1 + m_2$.
Furthermore, there are atomic terms: generators $g\in G$ (from a given
set $G$) of fixed type and four different constants
(Fig.~\ref{fig:constants}): $\id\colon 1\to 1$ (identity),
$\nabla\colon 1\to 2$ (duplicator), $\sigma\colon 2\to 2$
(permutation) and $\top\colon 1\to 0$ (terminator). There are also
derived constants of higher arity that are defined in
Table~\ref{tab:axioms-cc-prop} (upper half).

We require that all the axioms in Table~\ref{tab:axioms-cc-prop}
(lower half) are satisfied. While PROPs always contain the constant
$\sigma$, $\nabla$ and $\top$ are specific to CC-structured
PROPs.
The freely generated CC-structured PROP can be obtained by taking
all terms obtained inductively from generators and constants via
sequential composition and tensor, quotiented by the axioms.

\begin{figure}[h]
  \begin{minipage} {0.48\linewidth}
    \begin{tikzpicture}
      \centering
      \draw[black,  thick] (4,0) rectangle ++(.5,.75) node[pos=.5] {$f'$};
      \draw[black,  thick] (3,0) rectangle ++(.5,.75) node[pos=.5] {$f$};
      \draw[black,  thick] (0,0) rectangle ++(1.25,.75) node[pos=.5] {$f;f'$};

      \draw[black,  very thick, double] (-.5,0.375) -- ++(0.5,0)node[pos=.5, above=-.05]{$n$};
      \draw[black,  very thick, double] (1.25,0.375) -- ++(0.5,0)node[pos=.5, above=-.05]{$m$};
      \draw[white,  thick] (1.75,0.375) -- ++(0.75,0) node[pos=.5, black]{$=$};
      \draw[black,  very thick, double] (2.5,0.375) -- ++(0.5,0)node[pos=.5, above=-.05]{$n$};
      \draw[black,  very thick, double] (3.5,0.375) -- ++(0.5,0)node[pos=.5, above=-.05]{$k$};
      \draw[black,  very thick, double] (4.5,0.375) -- ++(0.5,0)node[pos=.5, above=-.05]{$m$};
    \end{tikzpicture}
  \end{minipage}
  \begin{minipage} {0.48\linewidth}
    \centering
    \begin{tikzpicture}
      \centering
      \draw[black,  thick] (3.75,-.5) rectangle ++(.5,.75) node[pos=.5] {$f_1$};
      \draw[black,  thick] (3.75,.5) rectangle ++(.5,.75) node[pos=.5] {$f_2$};
      \draw[black,  thick] (-.5,0) rectangle ++(1.5,.75) node[pos=.5] {$f_1 \otimes f_2$};

      \draw[black,  very thick, double] (-1.5,0.375) -- ++(1,0)node[pos=.3, above=-.05]{$n_1 + n_2$};
      \draw[black,  very thick, double] (1,0.375) -- ++(1,0)node[pos=.8, above=-.05]{$m_1 + m_2$};
      \draw[white,  thick] (2,0.375) -- ++(1.25,0) node[pos=.65, black]{$=$};
      \draw[black,  very thick, double] (3.25,.875) -- ++(0.5,0)node[pos=.5, above=-.05]{$n_1$};
      \draw[black,  very thick, double] (4.25,0.875) -- ++(0.5,0)node[pos=.65, above=-.05]{$m_1$};
      \draw[black,  very thick, double] (3.25,-.125) -- ++(0.5,0)node[pos=.5, above=-.05]{$n_2$};
      \draw[black,  very thick, double] (4.25,-.125) -- ++(0.5,0)node[pos=.65, above=-.05]{$m_2$};
    \end{tikzpicture}
  \end{minipage}
  \caption{String diagrammatic representation of the operators
    $\otimes$ and $;$ within a graph. Thick double lines represent
    several wires. The types are $f\colon n \rightarrow k$,
    $f'\colon k \rightarrow m$, $f_1\colon n_1 \rightarrow
    m_1$, $f_2\colon n_2 \rightarrow m_2$.}
  \label{fig:Kronecker-Multiplication}
\end{figure}
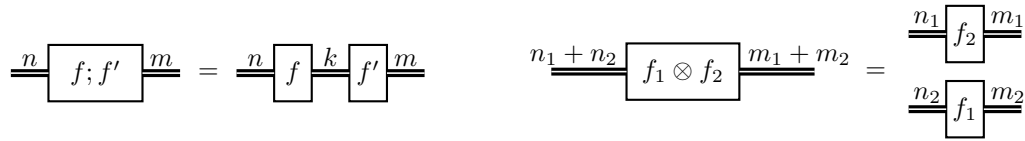

\begin{table}[h]
  \centering
  \cbox{
    \vspace{-0.4cm}
    \begin{eqnarray*}
      && \id_1 = \id \qquad \id_{n+1} = \id_n\otimes \id_1 \qquad\qquad
      \top_1 = \top\qquad \top_{n+1} = \top_n\otimes \top \\
      && \sigma_{n,0} = \sigma_{0,n} = \id_n \qquad 
      \sigma_{n+1,1} = (\id\otimes \sigma_{n,1});(\sigma\otimes \id_n)
      \\
      && \qquad \sigma_{n,m+1} = (\sigma_{n,m}\otimes \id_1);(\id_m\otimes
      \sigma_{n,1}) \\
      && \nabla_1 = \nabla \qquad \nabla_{n+1} = (\nabla_n\otimes
      \nabla);(\id_n \otimes \sigma_{n,1}\otimes \id)
    \end{eqnarray*}
    \vspace{-0.5cm}
    \hrule
    \vspace{-0.3cm}
    \begin{eqnarray*}
      && (t_1;t_3) \otimes (t_2;t_4) = (t_1\otimes t_2);(t_3\otimes
      t_4) \qquad
      (t_1;t_2);t_3 = t_1;(t_2;t_3) \\
      && \id_n;t = t = t;\id_m \qquad
      (t_1\otimes t_2)\otimes t_3 = t_1\otimes(t_2\otimes t_3) \qquad
      \id_0 \otimes t = t = t\otimes \id_0 \\
      && \sigma;\sigma = \id_2 \qquad
      (t\otimes\id_k);\sigma_{m,k} = \sigma_{n,k};(\id_k\otimes t)
      \qquad
      \nabla;(\nabla\otimes \id_1) = \nabla;(\id_1\otimes\nabla)
      \\
      && \nabla = \nabla;\sigma  \qquad 
      \nabla;(\id_1\otimes \top) = \id_1 
    \end{eqnarray*}
    \vspace{-0.7cm}
    \mbox{}
  }
  \caption{Operators
    of higher arity (above) and axioms for CC-structured PROPs (below)}
  \label{tab:axioms-cc-prop}
\end{table}

\begin{figure}[H] 
  \centering
  \begin{tikzpicture}[yscale=1]
    \draw[white,  thick] (-1,0) -- ++(0.75,0) node[pos=.5, black]{$\id =$};
    \draw[black,  thick] (0,0) -- ++(1,0);

    \draw[white,  thick] (2,0) -- ++(0.75,0) node[pos=.5, black]{$\nabla =$};
    \draw[black,  thick] (3,0) -- ++(0.5,0);
    \draw[black,  thick] (3.5,0) -- ++(0.5,0.5);
    \draw[black,  thick] (3.5,0) -- ++(0.5,-0.5);

    \draw[white,  thick] (5,0) -- ++(0.75,0) node[pos=.5, black]{$\sigma =$};
    \draw[black,  thick] (6,-0.5) -- ++(0.5,0);
    \draw[black,  thick] (6,0.5) -- ++(0.5,0);
    \draw[black,  thick] (6.5,0.5) -- ++(1,-1);
    \draw[black,  thick] (6.5,-0.5) -- ++(1,1);
    \draw[black,  thick] (7.5,-0.5) -- ++(0.5,0);
    \draw[black,  thick] (7.5,0.5) -- ++(0.5,0);


    \draw[white,  thick] (9,0) -- ++(0.75,0) node[pos=.5, black]{$\top =$};
    \draw[black,  thick] (10,0) -- ++(1,0);
    \draw[black,  thick] (11,.25) -- ++(0,-0.5);
  \end{tikzpicture}
  \caption{String diagrammatic equivalents to the
    constants}
  \label{fig:constants}
\end{figure}
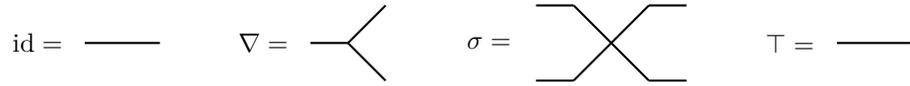

\medskip

Causality graphs (Def.~\ref{def:CG}) form a PROP, in fact the free
CC-structured PROP. They can simply be seen as the string diagram
representation of the arrows of the PROP. The axioms in
Fig.~\ref{tab:axioms-cc-prop} basically describe how to rearrange a
string diagram into an isomorphic one. The constants correspond to
causality graphs as drawn in Fig.~\ref{fig:constants}.

Hence, every causality graph has both a string diagrammatic
representation and is represented by an equivalence class of terms.
For instance the causality graph $B'$ in
Fig.~\ref{fig:bn-variable-elimination} can also be written as a term
$(A\otimes B\otimes C);(D\otimes \id_1);E$, where $A,B,C,D,E$ are the
corresponding matrices (given by $\ev$).

\medskip

Another instance of a CC-structured PROP are (sub-)stochastic
matrices, with entries taken from the closed interval
$[0,1] \subset \mathbb{R}$.

Here, the constants correspond to the following matrices:
\begin{eqnarray*}
  \id_0 = (1) \quad \id =
  \begin{pmatrix}
    1 & 0 \\ 0 & 1
  \end{pmatrix} \quad \nabla =
  \begin{pmatrix}
    1 & 0 \\ 0 & 0 \\ 0 & 0 \\ 0 & 1
  \end{pmatrix} \quad \sigma =
  \begin{pmatrix}
    1 & 0 & 0 & 0 \\ 0 & 0 & 1 & 0 \\ 0 & 1 & 0 & 0 \\ 0 & 0 & 0 & 1
  \end{pmatrix} \quad \top =
  \begin{pmatrix}
    1 & 1
  \end{pmatrix} 
\end{eqnarray*}

We index matrices over $\{0,1\}^m \times \{0,1\}^n$, i.e. for
$\X \in \{0,1\}^m$, $\Y \in \{0,1\}^n$ the corresponding entry
is denoted by $P(\X \mid \Y)$. 
%
%
The order of rows and columns in the matrix regarding the assignment
of events taking place or not, is descending (see matrix below).

\begin{wrapfigure}{r}{0.282\textwidth}
  \begin{center}
    ~\\[-0ex]
    $
    \begin{matrix}
      11 \\ 10 \\ 01 \\ 00
    \end{matrix}
    \begin{pmatrix}
      \    \smash{\raisebox{3ex}{\makebox[0pt]{11}}}
      {\cdot}\  &
      \ {\cdot}\smash{\raisebox{3ex}{\makebox[0pt]{10}}}
      \  & \ {\cdot}\smash{\raisebox{3ex}{\makebox[0pt]{01}}}
      \  & \ {\ \cdot}\smash{\raisebox{3ex}{\makebox[0pt]{00}}}\ \ \  \\
      \ \cdot\  & \ \cdot\  & \ \cdot\  & \ \cdot\ \ \\
      \ \cdot\  & \ \cdot\  & \ \cdot\  & \ \cdot\ \ \\
      \ \cdot\ & \ \cdot\ & \ \cdot\ & \ \cdot\ \
    \end{pmatrix}$\vspace{-2ex}
  \end{center}
\end{wrapfigure}

Sequential composition is \emph{matrix multiplication}, i.e., given
$P\colon n\to m$, $Q\colon m\to \ell$ we define
$P;Q = Q\cdot P\colon n\to \ell$, which is a
$2^\ell\times 2^n$-matrix. The tensor is given by the \emph{Kronecker
  product}, i.e., given $P\colon n_1\to m_1$, $Q\colon n_2\to m_2$ we
define $P\otimes Q\colon n_1+n_2\to m_1+m_2$ as
$(P\otimes Q)(\bitvec{x}_1\bitvec{x}_2\mid \bitvec{y}_1\bitvec{y}_2) =
P(\bitvec{x}_1\mid \bitvec{y}_1)\cdot Q(\bitvec{x}_2\mid
\bitvec{y}_2)$ where $\bitvec{x}_i \in \{0,1\}^{m_i}$,
$\bitvec{y}_i \in \{0,1\}^{n_i}$.

\medskip

A modular Bayesian network (MBN) is a causality graph, where every
generator is intepreted by a (sub-)stochastic matrix via an evaluation
function $\ev$. Given $\ev$, we obtain a mapping $M_\ev$ that
transforms causality graphs into (sub-)stochastic matrices
(Def.~\ref{def:eval}).  Note that $M_\ev$ is compositional, it
preserves constants, generators, sequential composition and
tensor. More formally, it is a functor between symmetric monoidal
categories, preserving also the CC-structure of the PROP. This also
means that evaluating a term via matrix operations and evaluating its
causality graph as described in Def.~\ref{def:eval}, gives the same
result.

\section{Modelling a Test with False Positives and Negatives}
\label{sec:test-fpfn}

\heading{Setup}

Here we analyze the net from Fig.~\ref{fig:pn-test} (discussed in
Ex.~\ref{ex:petri-nets}) that models a test with false positives and
negatives in more detail. We have a random variable $T$ that describes
whether the test is positive ($1$) or negative ($0$).

Remember that there are is one place ($I$, marked when the person is
infected) and three transitions, modelling the following events :
\begin{itemize}
\item $\mathit{flp}$: this transition has no pre- and no
  post-condition.  In the case where $I$ in unmarked, firing this
  transition denotes a false positive. It may also fire if $I$ is
  marked, in which case it stands for a test that actually went wrong
  and is only positive because of luck (lucky positive). Hence
  $p_T(\mathit{flp}) = P(R\mid \bar{I})$.
\item $\mathit{inf}$: this transition has $I$ in its pre- and
  post-condition. Hence it can be fired only if $I$ is marked,
  representing a test that truly -- and not by chance -- uncovers an
  infection. Hence
  $p_T(\mathit{inf}) = P(R\mid I) - P(R\mid \bar{I})$, the
  probability that has to be added to the probability of $\mathit{flp}$
  to obtain the true positive (assuming that the probability for false
  positive is less than the one for true positive).
\item $\failop$: the special fail transition stands for a failed test
  and its probability is $p_T(\failop) = P(\bar{R}\mid I)$ (false
  negative). Note that if the person is not infected also
  $\mathit{inf}$ fails and the sum of the probabilities is
  $p_T(\mathit{inf}) + p_T(\failop) = P(\bar{R}\mid I) + P(R\mid I)
  - P(R\mid \bar{I}) = 1 - P(R\mid \bar{I}) = P(\bar{R}\mid \bar{I})$ (true
  negative), exactly as required.
\end{itemize}

As explained before, we use the independent semantics.

\heading{Success Case}

Now we perform uncertainty reasoning as described in
Sct.~\ref{sec:uncertainty-reasoning} and the initial probability
distribution is
\[ p^0 =
  \begin{pmatrix}
    P(I) \\ P(\bar{I})
  \end{pmatrix} \]
Assume first that we observe success. In this case we have to multiply
$p^0$ with the matrix $P_*$ that is given as follows. 
\[ P_* =
  \begin{pmatrix}
    p_T(\mathit{flp})+p_T(\mathit{inf}) & 0 \\
    0 & p_T(\mathit{flp})
  \end{pmatrix} =
  \begin{pmatrix}
    P(R\mid I) & 0 \\
    0 & P(R\mid \bar{I})
  \end{pmatrix}\]
The first row/column always refers to the case where
$I$ is marked ($I$) and the second row/column to the case where $I$ is
unmarked ($\bar{I}$). Hence, the first entry in the diagonal gives us
the probability of going from marking $I$ to itself and the second
entry the probability of staying in the empty marking.

Hence by multiplying $P_*\cdot p^0$ we obtain
\[
  P_*\cdot p^0 = 
  \begin{pmatrix}
    P(R\mid I) & 0 \\
    0 & P(R\mid \bar{I})
  \end{pmatrix}\cdot
  \begin{pmatrix}
    P(I) \\ P(\bar{I})
  \end{pmatrix} =
  \begin{pmatrix}
    P(R\mid I)\cdot P(I) \\
    P(R\mid \bar{I})\cdot P(\bar{I})
  \end{pmatrix}
   = 
  \begin{pmatrix}
    P(R \cap I) \\
    P(R \cap \bar{I})
  \end{pmatrix}
\]
We normalize by dividing by
$P(R \cap I) + P(R \cap \bar{I}) = P(R)$ and get, using again
the definition of conditional probability:
\begin{eqnarray*}
  \norm(P_*\cdot p^0) & = &
  \begin{pmatrix}
    P(I\mid R) \\
    P(\bar{I}\mid R)
  \end{pmatrix}
\end{eqnarray*}

\heading{Failure Case}

We now switch to the case where failure is observed. In this case we
have to multiply $p^0$ with the matrix $F_*$ that is given as follows:
\[ F_* =
  \begin{pmatrix}
    p_T(\failop) & 0 \\
    0 & p_T(\mathit{inf})+p_T(\failop)
  \end{pmatrix} =
  \begin{pmatrix}
    P(\bar{R}\mid I) & 0 \\
    0 & P(\bar{R}\mid \bar{I})
  \end{pmatrix}\] The first entry in the diagonal gives us the
probability of failing from marking $I$ and the second entry
the probability of failing from the empty marking

Hence by multiplying $F_*\cdot p^0$ we obtain
\[
  F_*\cdot p^0 = 
  \begin{pmatrix}
    P(\bar{R}\mid I) & 0 \\
    0 & P(\bar{R}\mid \bar{I})
  \end{pmatrix}\cdot
  \begin{pmatrix}
    P(I) \\ P(\bar{I})
  \end{pmatrix} = 
  \begin{pmatrix}
    P(\bar{R}\mid I)\cdot P(I) \\
    P(\bar{R}\mid \bar{I})\cdot P(\bar{I})
  \end{pmatrix}
   = 
  \begin{pmatrix}
    P(\bar{R} \cap I) \\
    P(\bar{R} \cap \bar{I})
  \end{pmatrix}
\]
We normalize by dividing by
$P(\bar{R} \cap I) + P(\bar{R} \cap \bar{I}) = P(\bar{R})$ and get, using again
the definition of conditional probability:
\begin{eqnarray*}
  \norm(F_*\cdot p^0) & = &
  \begin{pmatrix}
    P(I\mid \bar{R}) \\
    P(\bar{I}\mid \bar{R})
  \end{pmatrix}
\end{eqnarray*}

Hence we obtain the probabilities that the person is infected
respectively not infected under the condition that the test is
positive respectively negative, exactly as required. This shows that
our formalism is expressive enough to model the standard testing
scenario with false postitives and negatives.

\section{Comparison to Term Width}
\label{sec:termwidth}

We here compare the notion of elimination to another notion of
width: term width, a very natural notion, since we are working in a
PROP.

Given a representation of the causality graph of a Bayesian network
$B$ as a term $t$, the width of $t$ is intuitively the size of the
largest matrix that occurs when evaluating $t$. In fact, an arrow of
type $m\to n$ corresponds to a matrix of dimensions $2^n\times 2^m$
where $2^n\cdot 2^m = 2^{n+m}$. As before, we will give the width or
size of the matrix as $m+n$.

\begin{defi}[Term width]
  Let $t$ be a term of a CC-structured PROP. Then we inductively
  define the \emph{term width} of $t\colon n\to m$, denoted by
  $\llangle t\rrangle$, as follows:
  \begin{itemize}
  \item Whenever $t$ is a generator or a constant ($\top$, $\sigma$,
    $\nabla$, $\mathit{id}$), the width of $t$, is the size of the
    corresponding matrix: $\llangle t\rrangle = n+m$.
  \item
    $\llangle t_1;t_2 \rrangle = \max\{\llangle t_1\rrangle, \llangle
    t_2\rrangle, n+m\}$.
  \item
    $\llangle t_1\oplus t_2 \rrangle = \max\{\llangle t_1\rrangle,
    \llangle t_2\rrangle, n+m\}$.
  \end{itemize}
  The term width of a causality graph $B$, denoted by
  $\llangle B\rrangle$, is the minimum width of a term that represents
  $B$.
\end{defi}

That is, we compute the sizes of all the matrices that we encounter
along the way and take the maximum of these sizes. As discussed
earlier, it is essential to choose a good term representation of a
Bayesian network in order to obtain small matrix sizes and hence an
efficient evaluation.


However, the elimination width and term width of a causality graph do
not necessarily coincide. This suggests that term width does not
provide suitable bounds for the actual computations. A direction of
future research is to come up with an alternative notion that better
relates the size of a term with the efficiency of its ``computation
recipe''.

However, we still obtain an upper bound for the elimination width,
by viewing every matrix as a factor.

\begin{restatable}{prop}{PropBNElimTreewidth}
  \label{prop:bn-elim-treewidth}
  There is a causality graph whose elimination width is strictly
  smaller than $\llangle B\rrangle$ and vice versa.
  Let $B$ be a causality graph of type $n\to m$. Then the elimination
  width of $B$ is bounded by $2\cdot \llangle B\rrangle$.
\end{restatable}


\section{Proofs}
\todo{\textbf{CONCUR Rev2:} In case the paper would be published without the 
appendix, I would be worried about its accessibility.}



\PropInterpretationProb*

\begin{proof}~
  
  \begin{itemize}

  \item $P(X_{n+1} = m' \mid X_{n+1}\neq *, X_n \neq *)$:

    Whenever $X_{n+1}\neq *$, we automatically have $X_n\neq *$, since
    one cannot leave the fail state. Hence:
    \begin{eqnarray*}
      && P(X_{n+1} = m' \mid X_{n+1}\neq *, X_n \neq *) = P(X_{n+1} = m' \mid
        X_{n+1}\neq *) \\
      & = & \frac{P(X_{n+1}=m \land X_{n+1}\neq *)}{P(X_n\neq *)} =
      \frac{P(X_{n+1}=m)}{\sum_{m'} P(X_{n+1}=m')} \\
      & = & \frac{p^{n+1}(m)}{\sum_{m'} p^{n+1}(m')} 
      = \norm(p^{n+1}_*)(m') = \norm((P^n\cdot
      p^n)_*)(m') = \norm(P^n_*\cdot p^n_*)(m')
    \end{eqnarray*}
  \item $P(X_n = m \mid X_{n+1} = *, X_n \neq *)$:

    We first observe that:
    \begin{eqnarray*}
      && P(X_n = m \land X_{n+1} = *) = P(X_n = m)\cdot P(X_{n+1}=*
      \mid X_n=m) \\
      & = & p^n(m) \cdot P^n(*\mid m) = p^n(m) \cdot F^n(m\mid
      m) = (F^n\cdot p^n)(m)
    \end{eqnarray*}

    From this we can derive:
    \begin{eqnarray*}
      && P(X_{n+1} = * \land X_n\neq *) = P(X_{n+1}=*\land
      \Big(\bigvee_{m'}
      X_n = m'\Big)) \\
      & = & P(\bigvee_{m'} (X_{n+1} = * \land X_n = m')) = \sum_{m'}
      P(X_{n+1} = * \land X_n = m') = \sum_{m'} (F^n\cdot p^n)(m')
    \end{eqnarray*}
    The second-last equality holds, since the events are disjoint. And
    so finally we obtain:
    \begin{eqnarray*}
      && P(X_n = m \mid X_{n+1} = *, X_n \neq *) = \frac{P(X_n=m \land
        X_{n+1} = *\land X_n \neq *)}{P(X_{n+1} = * \land
        X_n\neq *)} \\
      & = & \frac{P(X_n=m \land X_{n+1} = *)}{P(X_{n+1} = * \land
        X_n\neq *)} = \frac{(F^n\cdot
        p^n)(m)}{\sum_{m'} (F^n\cdot p^n)(m')} \\
      & = & \norm((F^n\cdot p^n)_*)(m') =
      \norm(F^n_*\cdot p^n_*)(m')
    \end{eqnarray*}
  \end{itemize}
\end{proof}

\PropPFMatrices*

\begin{proof}
  Let $m,m'$ be two markings which split into $m=m_1m_2$,
  $m'=m'_1m'_2$.
  \begin{itemize}
  \item $P^n_* = P'\otimes \Id_{2^{k-\ell}}$:
    
    \begin{eqnarray*}
      && P^n_*(m'\mid m) = \sum_{t, m\stackrel{t}{\Rightarrow} m'}
      r_n(m,t) = \sum_{t\in\bar{T}, m\stackrel{t}{\Rightarrow} m'}
      r_n(m,t) = \sum_{t\in\bar{T}, m\stackrel{t}{\Rightarrow} m'}
      \bar{r}(m_1,t) \\
      & = & \sum_{t\in\bar{T}, m_1\stackrel{t}{\Rightarrow} m'_1}
      \bar{r}(m_1,t)\cdot \Id_{2^{k-\ell}}(m'_2\mid m_2) \\
      & = & \Big(\sum_{t\in\bar{T}, m_1\stackrel{t}{\Rightarrow}
        m'_1} \bar{r}(m_1,t)\Big)\cdot \Id_{2^{k-\ell}}(m'_2\mid m_2) 
      = P'(m'_1\mid m_1) \cdot
      \Id_{2^{k-\ell}}(m'_2\mid m_2) \\
      & = & (P'\otimes \Id_{2^{k-\ell}})(m'\mid m)
    \end{eqnarray*}
    The second equality holds because only transitions of $\bar{T}$ can
    fire in $m$. The fourth equality is true since there is a
    transition $m\stackrel{t}{\Rightarrow} m'$ if and only if
    $m_1\stackrel{t}{\Rightarrow} m'_1$ and $m_2 = m'_2$.
  \item $F^n_* = F'\otimes \Id_{2^{k-\ell}}$: Here we
    distinguish two cases: if $m=m'$ then
    
    \begin{eqnarray*}
      && F^n_*(m'\mid m) = \sum_{t\in T_f, m\not
        \stackrel{t}{\Rightarrow}} r_n(m,t) = \sum_{t\in\bar{T},
        m\not\stackrel{t}{\Rightarrow}} r_n(m,t) = \sum_{t\in\bar{T},
        m\not\stackrel{t}{\Rightarrow}}
      \bar{r}(m_1,t)\cdot 1 \\
      & = & \sum_{t\in\bar{T}, m_1\not \stackrel{t}{\Rightarrow}}
      \bar{r}(m_1,t)\cdot \Id_{2^{k-\ell}}(m'_2\mid m_2)
      = F'(m'_1\mid m_1)\cdot \Id_{2^{k-\ell}}(m'_2\mid m_2) \\
      & = & (F'\otimes \Id_{2^{k-\ell}})(m'\mid m)
    \end{eqnarray*}
    
    In the other case ($m\neq m'$) we have
    \[
      F^n_*(m'\mid m) = 0 = F'(m'_1\mid m_1) \cdot
      \Id_{2^{k-\ell}}(m'_2\mid m_2) = (F'\otimes
      \Id_{2^{k-\ell}})(m'\mid m)
    \]
    Note that the second equality holds since whenever $m\neq m'$ we have
    $m_1\neq m'_1$ (and so $F'(m'_1\mid m_1) = 0$) or $m_2\neq m'_2$ (and so
    $\Id_{2^{k-\ell}}(m'_2\mid m_2) = 0$).
  \end{itemize}
\end{proof}

\PropVariableElimination*

\begin{proof}
  We first show that we obtain the correct result. For this, we define
  the subprobability distribution $p_\mathcal{F}$ associated to a
  multiset of factors $\mathcal{F}$. Let
  $\mathit{IW}_B = \{w_1,\dots,w_k\}$ be the set of internal wires
  and we fix the elimination ordering $w_1,\dots,w_k$.

  Then we define $p_\mathcal{F}\colon \{0,1\}^{n+m}\to [0,1]$ with:
  \[ p_\mathcal{F}(\X\Y) = \sum_{\Z\in\{0,1\}^k}
    \prod_{(f,\tilde{w}^f)\in \mathcal{F}}
    f(b_{\tilde{\iota}\tilde{o}\tilde{w},\X\Y\Z}(\tilde{w}^f)) \]
  where $\X\in \{0,1\}^n$, $\Y\in\{0,1\}^m$,
  $\tilde{\iota} = i_1\dots i_n$ and
  $\tilde{o} = \out(o_1)\dots \out(o_m)$. Clearly $p_{\mathcal{F}_0}$
  corresponds to $M_\ev(B)$ (see Def.~\ref{def:eval}), that is
  $p_{\mathcal{F}_0}(\X\Y) = M_\ev(B)(\Y\mid\X)$. Furthermore the
  result $p$ of the algorithm equals $p_{\mathcal{F}_k}$, since in
  this case $\Z$ is empty.

  We only have to show that in each step
  $p_{\mathcal{F}_{i-1}} = p_{\mathcal{F}_i}$. This can be seen by
  observing that -- due to distributivity -- we have:
  \begin{eqnarray*}
    && p_{\mathcal{F}_{i-1}}(\X\Y) \\
    & = & \sum_{\z_k\in\{0,1\}} \dots \sum_{\z_i\in\{0,1\}}
    \prod_{(f,\tilde{w}^f)\in \mathcal{F}_{i-1}}
    f(b_{i-1}(\tilde{w}^f)) \\
    & = & \sum_{\z_k\in\{0,1\}} \dots \sum_{\z_i\in\{0,1\}}
    \Big(\prod_{(h,\tilde{w}^h)\in \mathcal{F}_{i-1}\backslash
      C_{w_i}(\mathcal{F}_{i-1})} h(b_{i-1}(\tilde{w}^h))\Big) \cdot \Big(
    \prod_{(g,\tilde{w}^g)\in C_{w_i}(\mathcal{F}_{i-1})}
    g(b_{i-1}(\tilde{w}^g)) \Big) \\
    & = & \sum_{\z_k\in\{0,1\}} \dots \sum_{\z_{i+1}\in\{0,1\}}
    \Big(\prod_{(h,\tilde{w}^h)\in \mathcal{F}_{i-1}\backslash
      C_{w_i}(\mathcal{F}_{i-1})} h(b_i(\tilde{w}^h))\Big) \\
    && \qquad\qquad\qquad\qquad \mathop{\cdot} \Big(
    \underbrace{\sum_{\z\in \{0,1\}} \prod_{(g,\tilde{w}^g)\in
        C_{w_i}(\mathcal{F}_{i-1})}
      g(b_i(\tilde{w}^g))}_{f(b_i(\tilde{w}^f))} \Big) \\
    & = & \sum_{\z_k\in\{0,1\}} \dots \sum_{\z_{i-1}\in\{0,1\}}
    \prod_{(f,\tilde{w}^f)\in \mathcal{F}_i} f(b_i(\tilde{w}^f)) \\
    & = & p_{\mathcal{F}_i}(\X\Y),
  \end{eqnarray*}
  where $(f,\tilde{w}^f)$ is the new factor, produced in step~$i$ of
  the algorithm.  Here we use functions
  $b_i = b_{\tilde{\iota}\tilde{o}w_{k-i+1}\dots
    w_k,\Y\X\z_{k-i+1}\dots\z_k}$
  
  
  \medskip

  Furthermore observe that for every factor
  $(f,\tilde{w}^f)\in \mathcal{F}_i$, all the wires in $\tilde{w}^f$
  are connected via an edge in $U_i$, i.e., all those wires are part
  of a clique. This means that the size of the factors is bounded by
  the elimination width. We show this by induction on $i$.
  \begin{itemize}
  \item $i=0$: $\mathcal{F}_0$ contains those factors that correspond
    to the generators originally contained in $B$. Each of these
    generators induces a factor $(f,\tilde{w}^f)$ and a clique
    containing all vertices in $\tilde{w}^f$ in $U_0$.
  \item $i\to i+1$: In step~$i$ we eliminate wire $w_i$ and produce a
    new factor $(f,\tilde{w}^f)$. In order to produce this new factor,
    we multiply factors of $\mathcal{F}_{i-1}$. Furthermore, we obtain
    a graph $U_i$ that contains a clique of all all vertices in
    $\tilde{w}^f$. For the factors that we keep, the corresponding parts
    of the graph are unchanged and hence the corresponding cliques
    remain.
  \end{itemize}
\end{proof}

\PropTreewidthBN*

\begin{proof}
  First, by comparing Def.~\ref{def:treewidth-bn} and the
  definition of treewidth for undirected graphs from the literature
  \cite{bk:treewidth-computations-I}, one observes that the treewidth
  of a causality graph $B$ coincides with the treewidth of its
  undirected clique graph $U_0$ constructed above.

  Furthermore, according to
  \cite[Theorem~6]{bk:treewidth-computations-I}, the fact that $U_0$
  has treewidth $k$ is equivalent to the existence of an elimination
  order as in Def.~\ref{def:elimination-order}, where the
  highest clique in the graph sequence $U_i$ is bounded by $k$.

  Furthermore, as described in \cite{bk:treewidth-computations-I},
  every elimination order gives rise to a tree decomposition of the
  same width. Hence elimination width provides an upper bound for
  treewidth.

  \begin{figure}
    \centering
    \scalebox{0.5}{\input{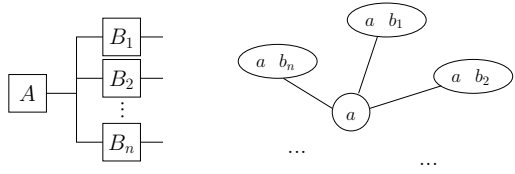}}
    \caption{A causality graph $B$ with $n$ output wires whose
      treewidth is strictly smaller than its elimination width}
    \label{fig:bn-treewidth}
  \end{figure}
  
  In order to show that the treewidth of a causality graph of type
  $0\to m$ may be strictly smaller than the elimination width,
  consider the network shown in Fig.~\ref{fig:bn-treewidth} (left): we
  only have to eliminate one internal wire, the wire $a$ that exits
  node $A$. By eliminating it, we obtain an $n$-clique of the output
  wires and thus we have elimination width~$n$. However, we have a
  star-shaped tree decomposition of width~$1$ (shown in
  Fig.~\ref{fig:bn-treewidth} on the right).
\end{proof}



\PropBNElimTreewidth*

\begin{proof}
  We first show that the elimination width of $B$ can be strictly
  smaller than $\llangle B\rrangle$, as observed by the following
  example (see Fig.~\ref{fig:bn-elim-treewidth}).

  The elimination width is $2$: if we denote the wires of the network
  by $a_1,a_2$ (the two wires originating from $A$) and $c$ (wire
  originating from $C$), then the joint distribution can be obtained
  as:
  \[ p(\a_1,\c) = \sum_{\a_2\in \{0,1\}} A(\a_1,\a_2)\cdot C(\c), \]
  where the largest factor that is involved is of size $2$.

  However, there is no way to represent this causality graph by a
  term $t$ of width $2$: we have to start with $A$ (of type $0\to 2$)
  and multiply it with some other matrix. Since we have to produce a
  matrix of type $0\to 2$ in the end, this other matrix has to have
  type $2\to 2$ (for instance $\id_1\otimes C$), which results in
  width $4$.\todo{Ba: check this argument. Can it be made more
    precise?}

  \medskip\hrule\medskip

  On the other hand, $\llangle B\rrangle$ can be strictly smaller than
  the elimination width of $B$, as observed by the causality graph in
  Fig.~\ref{fig:bn-tree-elim-width}, where we compose (multiply) two
  matrices $A,C$ of type $k\times k$.

  The largest matrix encountered during the computation is hence of
  size $k+k = 2k$. On the other hand, the width of the elimination
  ordering is $3k-1$: if we eliminate one of the inner wires, we
  immediately obtain a clique containing the remaining $3k-1$ vertices in
  $U_1$ (since every inner wire is connected to all the other wires).

  \begin{figure}
    \captionsetup[subfigure]{justification=centering}
    \begin{subfigure}[b]{0.5\textwidth}
      \centering \scalebox{0.5}{\input{bn-elim-tree-width-x.tex}}
      \caption{A causality graph $B$ whose elimination width is
        strictly smaller than $\llangle B\rrangle$}
      \label{fig:bn-elim-treewidth}
    \end{subfigure}
    \begin{subfigure}[b]{0.5\textwidth}
      \centering
      \scalebox{0.5}{\input{bn-tree-elim-width-x.tex}}
      \caption{A causality graph $B$ where $\llangle B\rrangle$ is
        strictly smaller than the elimination width}
      \label{fig:bn-tree-elim-width}
    \end{subfigure}
    \caption{}
  \end{figure}

  \medskip\hrule\medskip
  
  Assume that $B$ is represented by a term $t$, where
  $\llangle t\rrangle \le \llangle B\rrangle$. From $t$ we inductively
  derive an elimination order $\mathit{eo}(t)$ for the inner wires of
  $B$:
  \begin{itemize}
  \item whenever $t$ is a generator or a constant, the elimination
    order is the empty sequence (since there are no inner wires).
  \item whenever $t = t_1\otimes t_2$, then
    $\mathit{eo}(t) = \mathit{eo}(t_1) \mathit{eo}(t_2)$ (the
    concatenation of the elimination orderings)
  \item whenever $t = t_1;t_2$, then let $\tilde{w}$ be the sequence
    of wires that becomes internal due to the composition. Then
    $\mathit{eo}(t) = \mathit{eo}(t_1) \mathit{eo}(t_2) \tilde{w}$.
  \end{itemize}
  This gives us an ordering of all the inner wires of $B$.

  Now, for a given evaluation map $\ev$, we compute $M_\ev(B)$ according
  to the elimination order $\mathit{eo}(t)$ and prove by structural
  induction on $t$ that the elimination width of $B$ is bounded by
  $2\cdot\llangle t\rrangle$. In fact we use a stronger induction
  hypothesis where we show in addition that in the multiset of factors
  that we obtain at the very end, every factor has size at most
  $\llangle t\rrangle$.
  \begin{itemize}
  \item whenever $t$ is a generator or a constant, we have that the
    elimination width corresponds to the size of the largest
    generator. Hence, for a generator $g$ we have that the elimination
    width equals $\llangle t\rrangle\le 2\cdot \llangle t\rrangle$,
    whereas for the other constants, we have elimination width
    $0 \le 2\cdot \llangle t\rrangle$.
  \item whenever $t = t_1\otimes t_2$, we know by induction hypothesis
    that the elimination width of the causality graph $B_i$,
    represented by $t_i$, is bounded by $2\cdot \llangle t_i\rrangle$
    and that in the multiset of factors that we obtain at the end,
    every factor has size at most $\llangle t_i\rrangle$.

    This means that for $B = B_1\otimes B_2$ the algorithm starts with
    the disjoint union of factors of $B_1$ and $B_2$. If we follow the
    elimination order $\mathit{eo}(t_1) \mathit{eo}(t_2)$, we first
    process the factors of $B_1$, followed by the factors of $B_2$. No
    factor that is produced exceeds
    $\max\{2\cdot \llangle t_1\rrangle, 2\cdot \llangle t_2\rrangle\}
    = 2\cdot \max\{\llangle t_1\rrangle, \llangle t_2\rrangle\} \le
    2\cdot \llangle t\rrangle$ and in the end we obtain factors whose
    size is bounded by
    $\max\{\llangle t_1\rrangle, \llangle t_2\rrangle\} \le \llangle
    t\rrangle$.
  \item whenever $t = t_1;t_2$ (where $t_1\colon n\to k$,
    $t_2\colon k\to m$), we know by induction hypothesis that the
    elimination width of the causality graph $B_i$, represented by
    $t_i$, is bounded by $2\cdot \llangle t_i\rrangle$ and that in the
    multiset of factors that we obtain at the end, every factor has
    size at most $\llangle t_i\rrangle$.

    This means that for $B = B_1;B_2$ the algorithm starts with the
    disjoint union of factors of $B_1$ and $B_2$. If we follow the
    elimination order $\mathit{eo}(t_1) \mathit{eo}(t_2)$, we first
    process the factors of $B_1$, followed by the factors of $B_2$. As
    above, no factor that is produced exceeds
    $\max\{2\cdot \llangle t_1\rrangle, 2\cdot \llangle t_2\rrangle\}
    \le 2\cdot \llangle t\rrangle$ and in the end we obtain factors
    whose size is bounded by
    $\max\{\llangle t_1\rrangle, \llangle t_2\rrangle\} \le \llangle
    t\rrangle$.

    Next, we have to eliminate the wires in $\tilde{w}$. The largest
    factor that can be produced in this way is of size
    $n+k+k+m-1 \le \llangle t_1\rrangle + \llangle t_2 \rrangle \le
    2\cdot \max\{\llangle t_1\rrangle, \llangle t_2\rrangle\} \le
    2\llangle t\rrangle$. At the very end, once we have eliminated all
    the wires, we obtain factors whose size is at most
    $n+m \le \llangle t\rrangle$.   
  \end{itemize}
\end{proof}

}

\end{document}